\documentclass{article}
\PassOptionsToPackage{numbers, compress}{natbib}
\usepackage{fullpage}
\usepackage[utf8]{inputenc}
\usepackage[T1]{fontenc}
\usepackage{hyperref}
\usepackage{url}
\usepackage{booktabs}
\usepackage{amsfonts}
\usepackage{nicefrac}
\usepackage{microtype}    
\usepackage{microtype}
\usepackage{graphicx}
\usepackage{subfigure}
\usepackage{natbib}
\usepackage{hyperref}
\usepackage[normalem]{ulem}
\usepackage[usenames,dvipsnames]{xcolor}
\usepackage{amsmath}
\usepackage{amssymb}
\usepackage{mathtools}
\usepackage{amsthm}
\usepackage[capitalize,noabbrev]{cleveref}
\usepackage{times}
\usepackage{microtype}
\usepackage{graphicx}
\usepackage{tikz}
\usepackage{amsmath}
\usepackage{amssymb}
\usepackage{mathtools}
\usepackage{mathrsfs}
\usepackage[normalem]{ulem}
\usepackage{amsmath,amsfonts,bm}
\usepackage{cancel}

\theoremstyle{plain}
\newtheorem{theorem}{Theorem}[section]

\newtheorem{lemma}[theorem]{Lemma}
\newtheorem{corollary}[theorem]{Corollary}

\theoremstyle{definition}
\newtheorem{definition}[theorem]{Definition}

\theoremstyle{remark}
\newtheorem{remark}[theorem]{Remark}

\newcommand{\E}{\mathcal{E}}

\let\cup\bigcup

\def\floor#1{\lfloor #1 \rfloor}

\def\vzero{{\bm{0}}}
\def\vone{{\bm{1}}}

\def\va{{\bm{a}}}
\def\vb{{\bm{b}}}

\def\ve{{\bm{e}}}

\def\vq{{\bm{q}}}

\def\vu{{\bm{u}}}
\def\vv{{\bm{v}}}
\def\vw{{\bm{w}}}
\def\vx{{\bm{x}}}

\def\mA{{\bm{A}}}
\def\mB{{\bm{B}}}

\def\mI{{\bm{I}}}

\def\mW{{\bm{W}}}

\DeclareMathAlphabet{\mathsfit}{\encodingdefault}{\sfdefault}{m}{sl}
\SetMathAlphabet{\mathsfit}{bold}{\encodingdefault}{\sfdefault}{bx}{n}

\def\gA{{\mathcal{A}}}

\def\gC{{\mathcal{C}}}

\def\gG{{\mathcal{G}}}
\def\gH{{\mathcal{H}}}

\def\gL{{\mathcal{L}}}

\def\gN{{\mathcal{N}}}

\def\gX{{\mathcal{X}}}

\def\E{{\mathbb{E}}}

\def\N{{\mathbb{N}}}

\def\P{{\mathbb{P}}}

\def\R{{\mathbb{R}}}
\def\bbS{{\mathbb{S}}}

\DeclareMathOperator*{\argmin}{arg\,min}

\DeclareMathOperator{\Unif}{Uniform}

\newcommand{\norm}[1]{\left\lVert{#1}\right\rVert}
\newcommand{\setofnns}[1]{\mathcal{N}_{#1}}

\newcommand{\xydist}{\mathscr{D}}
\newcommand{\emloss}[2]{\mathscr{L}_{#1}\left(#2\right)}
\newcommand{\poploss}{\mathscr{L}}
\newcommand{\Repregbias}[1]{R_{#1}}

\newcommand{\realrule}[1]{\ensuremath{\gA^{\theta}_{#1}}}
\newcommand{\alpharealrule}[1]{\gA^{\theta,\alpha}_{#1}}
\newcommand{\idealrule}[1]{\gA^{*}_{#1}}
\newcommand{\Rad}{\mathscr{R}}
\newcommand{\width}{\omega}
\newcommand{\minwidth}{\width_{0}}
\newcommand{\twolayerwidth}{\width_{2}}
\newcommand{\pareto}[1]{\mathcal{P}_{#1}(S)}
\newcommand{\gNbar}{\gN}

\PassOptionsToPackage{unicode}{hyperref}
\PassOptionsToPackage{naturalnames}{hyperref}

\bibpunct{(}{)}{;}{a}{,}{,}

\usepackage{titlesec}
\titlespacing*{\section}{0pt}{3pt}{3pt}
\titlespacing*{\subsection}{0pt}{6pt}{3pt}
\titlespacing*{\paragraph}{0pt}{0pt}{3pt}

\title{Depth Separation in Norm-Bounded Infinite-Width Neural Networks}

\author{Suzanna Parkinson\footnote{Committee on Computational and Applied Mathematics, University of Chicago, Chicago, IL, USA.}, Greg Ongie\footnote{Department of Mathematical and Statistical Sciences, Marquette University, Milwaukee, WI, USA.}, Rebecca Willett\footnote{Department of Statistics and Department of Computer Science, University of Chicago, Chicago, IL, USA.}, Ohad Shamir\footnote{Department of Computer Science and Applied Mathematics, Weizmann Institute of Science,  Rehovot, Israel.}, Nathan Srebro\footnote{Toyota Technological Institute at Chicago,  Chicago, IL, USA.}}

\begin{document}
\maketitle

\begin{abstract}
We study depth separation in infinite-width neural networks, where complexity is controlled by the overall squared $\ell_2$-norm of the weights (sum of squares of all weights in the network). 
Whereas previous depth separation results focused on separation in terms of width, such results do not give insight into whether depth determines if it is possible to learn a network that generalizes well even when the network width is unbounded.  
Here, we study separation in terms of the sample complexity required for learnability. Specifically, we show that there are functions that are learnable with sample complexity polynomial in the input dimension by norm-controlled depth-3 ReLU networks, yet are not learnable with sub-exponential sample complexity by norm-controlled depth-2 ReLU networks (with any value for the norm). We also show that a similar statement in the reverse direction is not possible:  any function learnable with polynomial sample complexity by a norm-controlled depth-2 ReLU network with infinite width is also learnable with polynomial sample complexity by a norm-controlled depth-3 ReLU network.
\end{abstract}

\section{Introduction}

It has long been postulated that in training neural networks, ``the
size of the weights is more important than the size of the network''
\citep{bartlett1996valid}. That is, the inductive bias and generalization properties of learning neural networks come from seeking networks with small weights (in terms of magnitude or some norm of the weights), rather than constraining the number of weights.  Small weight norm is sufficient to ensure generalization \citep[e.g.][]{bartlett2002rademacher,neyshabur2015norm,golowich2018size,du2018power,daniely2019generalization}, and may be induced either through
explicit regularization \citep[e.g.,~via weight decay][]{hanson1988comparing} or
implicitly through the optimization algorithm \cite[e.g.][]{neyshabur2014search,neyshabur2017geometry,chizat2020implicit,vardi2023implicit}. The
reliance on weight-norm-based complexity control is particularly relevant with modern, heavily
overparameterized networks, which have more weights than training
examples. These networks can shatter the training set, and hence the size of the network alone does not lead to meaningful generalization guarantees
\citep{zhang2017understanding,neyshabur2014search}.  Indeed, over the years there has been increasing interest in the theoretical study of learning with {\em infinite width} networks, where the number of units per layer is unbounded or even infinite, while controlling the {\em norm} of the weights 
\citep{cho2009kernel,neyshabur2015norm,bach2017breaking,bengio2005convex,mei2019mean,chizat2018global,jacot2018neural,savarese2019infinite,ongie2019function,chizat2020implicit,pilanci2020neural,parhi2021banach,unser2023ridges}.

Considering infinite-width neural networks, and relying only on the norm of the weights for inductive bias and generalization, also requires a fresh look at the role of depth. The traditional study of the role of depth focused on how deeper networks can
represent functions using fewer units.
\cite[e.g.][]{pinkus1999approximation,telgarsky2016benefits,eldan2016power,liang2016deep,lu2017expressive,daniely2017depth,safran2017depth,yarotsky2017error,yarotsky2018optimal,rolnick2018power,arora2018understanding,safran2019depth,vardi2020neural,chatziafratis2020better,venturi2022depth}. Focusing on depth-2 (one hidden layer) versus depth-3 (two hidden layers) feedforward
neural networks with ReLU activations (see \Cref{sec:representation cost def} for precise
details),  traditional depth separation results tell us that there are
functions that can be well-approximated using depth-3, low-width networks (number of neurons polynomial in the input dimension), but cannot be approximated using depth-2 networks unless the width/number of neurons is exponentially high in the input
dimension. However, this separation is not relevant when studying infinite-width networks.

Instead of studying depth separation in terms of the {\em number of
weights} (i.e., width), one can study depth separation in terms of the
{\em size of the weights}, i.e., the norm required to approximate the
target function with a specific depth.  This is captured by the {\em
representation} cost $R_L(f)$, which is the minimal weight norm
(sum of squares of all weights in the network) required to represent
$f$ using an unbounded-width depth-$L$ network. One can ask whether there are functions that can be well approximated with a low $R_3$ representation cost, but which require a high $R_2$ representation cost to approximate, even if we allow unbounded or infinite width. One contribution of our paper is to show that the answer is ``yes'': the same function families that show depth separations in terms of width also demonstrate depth separations in terms of norm or representation cost. Specifically, with depth-3 networks, one can approximate functions in these families with norm polynomial in the input dimension, but with depth-2 networks, even with infinite width, an exponential norm is required to approximate functions in these families even within constant approximation error. At a technical level, this argument follows from explicitly accounting for the norm in the depth-3 representation, and by showing through a Barron-like unit-sampling argument that if such ``hard'' functions were approximable with a low norm in depth 2, they would also be approximable with a small width in depth 2, which we know from the width-based depth separation results is not true. 

What does such separation between $R_3$ 
and $R_2$ representation cost tell us?  
Without further analysis of the effect of this separation on learning capabilities, it is unclear.
One cannot directly compare the values of $R_2$ and $R_3$ since their comparison depends on the precise way we aggregate the norms across layers; see, e.g., \cite{neyshabur2015norm} for a careful discussion. While width-separation results can be thought of as a separation in terms of the required memory costs, when discussing infinite networks we are already abstracting away the computational implementation, and working with exponentially large {\em weights} is not an inherent computational barrier as the number of bits is still polynomial.

Thus, instead of studying depth separation in terms of {\em approximation}, we directly study the separation in terms of {\em learning}, as captured by its effect on sample complexity.  We ask
the following question: If Alice is learning using norm-based
inductive bias (i.e., regularization) with unbounded-width depth-2
networks, and Bob is learning using norm-based inductive bias with
unbounded-width depth-3 networks, are there functions Bob can learn
with a small {\em number of samples}, but which Alice would require a huge {\em number of samples} to learn? On the other hand, are there perhaps functions for which depth-2 would be better, i.e., which Alice can learn with a small number of samples with depth-2, but for which Bob would require a huge number of samples to learn by seeking a low-norm depth-3 network?  As formalized in \Cref{sec:learning rule def}, we think of Alice and Bob as using a standard Regularized
Empirical Risk Minimization or Structural Risk Minimizing (SRM)
approach, where they learn by minimizing some combination of the
empirical loss $\emloss{S}{f}$ and weight norm, or equivalently
representation cost $R_L(f)$, for depth $L=2$ or depth $L=3$.

Our main results are as follows (where we focus on learning functions with samples from a particular  distribution chosen for technical convenience):
\begin{theorem}(Depth Separation, Informal)
    \label{thm:depth separation informal}
There is a family of functions $f_d:\mathbb{R}^{2d}\rightarrow\mathbb{R}$ that requires exponential (in $d$) sample complexity to learn to within constant error by regularizing the norm in an unbounded width depth-2 ReLU network, but which can be learned with $poly(d,1/\varepsilon)$ samples to within any error $\varepsilon$ by regularizing the norm in a depth-3 ReLU network.
\end{theorem}
The next result ensures that the reverse of \Cref{thm:depth separation informal} does not occur.
\begin{theorem}(No Reverse Depth Separation, Informal)
    \label{thm:no reverse depth separation informal}
    Any function learnable with $poly(d,1/\varepsilon)$ samples by regularizing the norm in an unbounded width depth-2 ReLU network, can also be learned with $poly(d,1/\varepsilon)$ samples by regularizing the norm in a depth-3 ReLU network.
\end{theorem}
From these results, we conclude that functions that are ``easy" to learn with depth-2 ReLU networks form a strict subset of the functions that are ``easy" to learn with depth-3 ReLU networks.

At a high level, the proof of \Cref{thm:depth separation informal} relies on 
choosing a target function that is not approximable by a small norm depth-2 network. 
We then construct a depth-2 interpolant whose representation cost depends only mildly on the number of samples. Using the Alice-and-Bob terminology from earlier, since Alice (who utilizes depth-2 networks) tries to find a function that fits the data well and has a small representation cost, the representation cost of her function will be at least as small as that of the interpolant. Hence, unless she has access to an enormous number of samples, her function will not be able to approximate the target and will not generalize.
However, the target function is approximable by a depth-3 network with a small representation cost, so the Rademacher complexity results of \cite{neyshabur2015norm} lead to sample complexity bounds that allow us to bound Bob's generalization error with many fewer samples. To prove \Cref{thm:no reverse depth separation informal}, we show using a similar argument that Alice can only learn if the $R_2$ cost of approximating the target is small. We show that functions with small $R_2$ cost also have small $R_3$ cost, and so
Bob must also be able to learn these target functions. 

We see our contributions here on two levels:
\begin{enumerate}
    \item Providing a detailed study of depth separation in neural networks in terms of the size of the weights rather than the number of the weights.
    \item Establishing a framework and template for studying depth separation, or model separation more broadly, directly in terms of learning, with the separation being between low and high sample complexity.  This is in contrast to a study solely in terms of 
    the ``complexity'' needed to approximate target functions, which does not directly provide insights into sample complexities. 
\end{enumerate}

\subsection{Outline}
We define the representation cost and describe its connection to weight decay regularization in \Cref{sec:representation cost def}. 
In \Cref{sec:dep sep in width vs rep cost} we consider depth separation in the norm to approximate certain families. 
We more carefully describe what we mean by learning rules using a norm-based inductive bias in \cref{sec:learning rule def}.
The formal statements of \Cref{thm:depth separation informal,thm:no reverse depth separation informal} are in \Cref{sec:main results}, and their proof sketches are in \Cref{sec:dep sep,sec:no rev dep sep proof}, respectively. 
We conclude in \Cref{sec:conclusion} with a discussion of the implications and limitations of these results.
All technical lemmas and their proofs are reserved for \Cref{sec:technical proofs}.

\subsection{Notation}
\label{sec:notation}
The set of depth-$L$ width-$\width$ ReLU neural networks is denoted as $\setofnns{L,\width}$, and the set of depth-$L$ unbounded-width networks is denoted as $\setofnns{L}:= \cup_{\width \in \N} \setofnns{L,\width}$.
We use $\bbS^{d-1}$ for the hypersphere in $\R^d$, and $\gX_d := \bbS^{d-1} \times \bbS^{d-1} \subseteq \R^{2d}$ to denote the Cartesian product of two hyperspheres. Given $\vx \in \gX_d$, we write $\vx^{(1)}$ and $\vx^{(2)}$ for the first and last $d$ entries in $\vx$, respectively.
Throughout the remainder of the paper, we assume that the dimension parameter $d$ is at least two.
We use $\|\cdot\|_{L^2}$ for the $L^2$ norm over $\gX_d$; that is,
$
    \|f\|_{L^2}^2 = \E_{\vx \sim \Unif(\gX_d)}[f(\vx)^2].
$
Similarly, we use $\|\cdot\|_{L^\infty}$ for the $L^\infty$ norm over $\gX_d$.
We write $\xydist_d$ for a distribution on $\gX_d \times [-1,1]$.
We use the squared error loss and write $\poploss_{\xydist_d}(f) = \E_{(\vx,y) \sim \xydist_d}[(f(\vx) - y)^2]$ for the generalization error of a model $f$. Given a sample $S = \{(\vx_i,y_i)\}_{i=1}^m$ of size $m$ drawn i.i.d. from $\xydist_d$, we denote the sample loss as $\emloss{S}{f} := \frac{1}{m} \sum_{i=1}^m(f(\vx_i) - y_i)^2$.

\section{Norm-Based Control in Infinite-Width Networks}
\label{sec:representation cost def}

In this work, we focus on the class of fully connected depth-$L$ neural networks with ReLU activations, $2d$-dimensional inputs, and scalar output (or a \emph{depth-$L$ network}, for short). A depth-$L$ network realizes a function $f_{\bm\phi}:\R^{2d}\rightarrow \R$ of the form:
\begin{equation*}
f_{\bm\phi}(\vx) = \vw_L^\top[\mW_{L-1}[\cdots [\mW_2[\mW_1\vx + \vb_1]_+ + \vb_{2}]_+ \cdots]_+ + \vb_{L-1}]_+ + b_L 
\end{equation*}
where $\bm\phi := (\mW_1,\vb_1,\ldots,\mW_{L-1},\vb_{L-1},\vw_L,b_L)$ denotes the collection of all weight matrices $\mW_\ell \in \R^{\width_{\ell}\times \width_{\ell-1}}$,
 bias vectors $\vb_\ell \in \R^{\width_{\ell}}$, plus outer layer weights  $\vw_L \in \R^{\width_{L-1}}$ and bias $b_L \in \R$,
and $[\cdot]_+$ denotes the ReLU activation applied entrywise. Here,  we allow the hidden-layer widths $\width_{\ell}$ for $\ell=1,...,L-1$ to be arbitrarily large. 

Let $\Phi_L$ denote the collection of all parameter vectors $\bm\phi$ associated with a depth-$L$ network, and define $\gN_L = \{f_{\bm\phi} : \bm \phi \in \Phi_L\}$ to be the space of all functions realized by a depth-$L$ network of unbounded width. Given a function $f \in \gN_L$, we define its depth-$L$ \emph{representation cost} $\Repregbias{L}(f)$ by
\begin{equation}
\label{eq:representation cost def}
\Repregbias{L}(f) = \inf_{\bm\phi \in \Phi_L :  f = f_{\bm \phi}} \frac{\|\bm\phi\|^2}{L}.
\end{equation}
where $\|\bm\phi\|^2$ denotes the sum of squares of all weights/biases in the network $f_{\bm\phi}$, and $f = f_{\bm \phi}$  indicates equality over the domain $\gX_d$. 
More generally, following \cite{savarese2019infinite, ongie2019function}, one can extend the definition of $R_L$ to a broader class of functions $f \in \gC(\gX_d)$ by
\begin{equation}
\Repregbias{L}(f) = \lim_{\epsilon\rightarrow 0}\, \inf \left\{ \frac{\|\bm\phi\|^2}{L} : \|f - f_{\bm \phi}\|_{L^\infty} \leq \epsilon, \bm\phi \in \Phi_L \right\}
\end{equation}
where $R_L(f) = +\infty$ if the limit above does not exist. Any function with $R_L(f) < +\infty$ and $f \notin \gN_L$ can be considered an ``infinite-width'' neural network, i.e., the uniform limit of a sequence depth $L$ networks with unbounded width whose representation cost remains bounded. 
Since we focus on the representation cost needed to approximate functions, it suffices to consider networks whose width is unbounded but finite. In this case, the definition in \eqref{eq:representation cost def} suffices.

The representation cost arises naturally when considering empirical risk minimization (ERM) with weight decay regularization:
\begin{equation}
    \label{eq:regphi}
\min_{{\bm\phi}\in \Phi_L} \emloss{S}{f_{\bm\phi}} + \frac{\lambda}{L} \|\bm\phi\|^2,
\end{equation}
where $\lambda > 0$ is a tunable regularization parameter.
By fixing a function $f \in \gN_L$ and optimizing over its parametrizations $f = f_{\bm\phi}$ as an $L$-layer network, we see that the above parameter space minimization problem is equivalent to the function space minimization problem
\begin{equation}
\min_{f\in {\gN}_{L}} \emloss{S}{f} + \lambda \Repregbias{L}(f).
\end{equation}
In other words, the representation cost is the function space regularization penalty induced by imposing weight decay regularization in parameter space.

\begin{remark}
    In \Cref{remark:bounded width results}, we consider generalizations of our results to bounded-width networks. In that case, it is useful to consider the bounded-width version of the representation cost, which is the natural analog of the weight decay penalty in the function space $\setofnns{L,\width}$ of functions realized by an $L$-layer network with the hidden-layer widths bounded by $\width$. In this case we write the representation cost as $\Repregbias{L}(f;\width)$, and we formally define
    \begin{equation}
        \Repregbias{L}(f;\width) := \inf_{\substack{\bm\phi \in \Phi_L :  f = f_{\bm \phi} \\ \width_1,\ldots,\width_{L-1} \le \width}} \frac{\|\bm\phi\|^2}{L}.
    \end{equation}
\end{remark}

To better understand the inductive bias of learning with weight decay regularization, several recent works have sought to give explicit function space characterizations of the representation cost $R_L(f)$.  First, \cite{savarese2019infinite} showed that, for univariate functions, and assuming unregularized bias terms, $R_2(f)$ coincides with the $L^1$-norm of the second derivative of the function. This was generalized to multidimensional inputs $(d>1)$ by \citet{ongie2019function}, where it is shown that $R_2(f)$ is equal to the $L^1$-norm of the Radon transform of a $(d+1)$-order derivative operator applied $f$. Related works have studied the impact of other activation functions \citep{parhi2020role}, multi-dimensional outputs \citep{shenouda2023vector} and regularizing bias terms \citep{boursier2023penalising}. An ongoing effort is to characterize $R_L(f)$ with depth $L>2$. For networks with multi-dimensional outputs, the limit as depth $L\rightarrow \infty$ is studied in \citep{jacot2022implicit}, where it is conjectured that the limiting representation cost coincides with the so-called ``bottleneck rank'' of a function, defined as the minimum $r$ such that $f = g\circ h$ with $h:\R^{d_\text{in}}\rightarrow\R^r$ and $g:\R^r\rightarrow \R^{d_{\text{out}}}$. Finite depth modifications to this characterization are also studied by \citet{jacot2023bottleneck}.
\section{Norm-Based Depth Separation in Approximation}
\label{sec:dep sep in width vs rep cost}
\label{sec:depth separation in norm}
Most previous depth separation results focus on separation in terms of the size of the network (i.e., the number of neurons) needed to represent or well-approximate a given target function.
Specifically, \citet{eldan2016power,daniely2017depth,safran2017depth} showed there are families of target functions parameterized by input dimension $d$ that are well-approximated by a depth-3 network whose number of neurons is polynomial in $d$, but require width exponential in $d$ to approximate within constant accuracy using a depth-2 network. 
For concreteness, we highlight the result from \cite{daniely2017depth}:
\begin{lemma}[\cite{daniely2017depth}]\label{lem:Daniely} There exists a family of functions $\{f_d\}_{d=1}^\infty \subset L^2(\gX_d)$ such that any depth-two ReLU network $f_{\bm\phi}\in \gN_2$ with $\|\bm\phi\|_\infty \leq 2^d$ satisfying
$
\|f_d-f_{\bm\phi}\|_{L^2} < 10^{-4}$
must have width $\width = 2^{\Omega(d\log(d))}$. Conversely, for any $\epsilon > 0$, there exist a depth-three ReLU network $\tilde{f}_{\bm\phi} \in \gN_3$ 
with $O(\mathrm{poly}(d)/\epsilon)$ neurons and $\|\bm\phi\|_\infty =O(\mathrm{poly}(d))$, such that $\|f_d - \tilde{f}_{\bm\phi}\|_{L^\infty} < \epsilon$.
\end{lemma}

However,  a width-based depth separation like the one above is not meaningful in the infinite-width setting. Instead, we consider whether a similar depth separation occurs in terms of the norm of the network (i.e., its representation cost). As a first result in this direction, \cite{ongie2019function} showed that there are functions in any input dimension $d$ with finite $R_3$ representation cost but infinite $R_2$ representation cost, in the sense that any sequence of depth-2 
networks converging pointwise to the target function on all of $\R^d$ must have unbounded representation cost. Yet, this left open whether there is still a depth separation in the representation costs required to \textit{approximate} the target to a given accuracy on a bounded domain, and if so, its dependence on input dimension $d$.
Here, we settle the question. In particular, we show the same function families that show depth separations in terms of width to approximate also demonstrate depth separations in terms of representation cost to approximate. 

A key tool in moving from separation in terms of width to separation in terms of representation cost is the following lemma, which says that depth-2 neural networks of any width can be well approximated by narrow networks having essentially the same representation cost (i.e., up to a small constant factor). The proof follows essentially the same sampling argument as in Barron's universal approximation theorem for depth-2 networks \citep{barron1993universal}; the details are given in Appendix \ref{sec:can approximate with narrow network with same R2 cost}. 
\begin{lemma}
\label{lem:can approximate with narrow network with same R2 cost}
For any $f \in \setofnns{2}$, $\varepsilon > 0$, and width $\width > \frac{3R_2(f)^2}{\varepsilon^2}$, there exists $f_{\bm\phi} \in \setofnns{2}$ having width $\width$ and $\norm{\bm\phi}^2_\infty \leq \norm{\bm\phi}^2_2 \leq 4 R_2(f)$ such that $\|f - f_{\bm\phi}\|_{L^2} \leq \varepsilon$.
\end{lemma}

Consider function families that we know require large widths to approximate with depth-2 networks, but can be well approximated with small width depth-3 networks with bounded weights. Functions in this family \textit{must} have large $R_2$ cost; otherwise, \Cref{lem:can approximate with narrow network with same R2 cost} 
would imply they can be approximated with a small width.
On the other hand, small width depth-3 networks with bounded weights must have low $R_3$ cost. In particular, a family of depth-3 networks whose width is $\mathrm{poly}(d)$ and weight magnitudes are $\mathrm{poly}(d)$ must have $R_3$ cost at most $\mathrm{poly}(d)$. Therefore, a depth separation in width to approximate should also imply a depth separation in representation cost to approximate.

Applying the above argument to the family of functions identified \Cref{lem:Daniely}, we arrive at the following result, which is proved in Appendix \ref{sec:can approximate with narrow network with same R2 cost}:
\begin{corollary}
\label{cor:depth separation in norm to approximate}
There exists a family of functions $\{f_d\}_{d=1}^\infty \subset L^2(\gX_d)$ such that each $f_d$ can be $\varepsilon$-approximated in $L^\infty$-norm by a depth-three network $\tilde{f}_d \in \gN_3$ with $R_3(\tilde{f}_d) = O(\mathrm{poly}(d)/\varepsilon)$, yet to approximate $f_d$ by a depth-two network $\hat{f}_d \in \gN_2$ to constant accuracy in $L^2$-norm requires $R_2(\hat{f}_d) = 2^{\Omega(d\log(d))}$.
\end{corollary}

While mathematically interesting, this type of norm-based depth separation in approximation does not immediately imply anything about learning with norm-controlled networks, e.g., whether there is also a depth separation in the sample complexity needed for good generalization. In the remainder of this paper, we close this gap and show that a norm-based depth separation in approximation also implies a depth separation in sample complexity for norm-based learning rules.

\section{Infinite-Width Norm-Based Learning Rules}
\label{sec:learning rule def}

We consider learning using the representational cost $R_L(f)$ as an inductive bias (i.e.,~complexity measure).  Following the Structural Risk Minimization principle, we consider learning rules minimizing some combination of the empirical risk $\emloss{S}{f}$ and the representational cost $R_L(f)$:
\begin{equation}   
\label{eq:general bicriterion minimization problem}
    \min_{f \in \gNbar_L} 
    \left(\emloss{S}{f}, R_L(f)\right).
\end{equation}
More specifically, we consider any learning rule returning a Pareto optimal point for the bi-criteria problem \eqref{eq:general bicriterion minimization problem}.  This includes any minimizer of the regularized risk 
    \begin{equation}
    \label{eq:L-layer RERM}
        \min_{f\in \gNbar_L} \emloss{S}{f} + \lambda R_L(f)
    \end{equation}
for any $\lambda > 0$, where recall that \eqref{eq:L-layer RERM} is equivalent to seeking an unbounded width network and regularizing the norm of the weights, as in \eqref{eq:regphi}.  We denote the set of all Pareto optimal points of \eqref{eq:general bicriterion minimization problem} (i.e.~the ``Pareto frontier'' or ``regularization path'', and including all minimizers of \eqref{eq:L-layer RERM}---see Figure \ref{fig:learning rules} for a visualization of the Pareto frontier and the learning rules considered)
by $\pareto{L}$. Similarly, we use $\pareto{L,\width}$ to denote the Pareto frontier of the bounded-width version of this problem:
\begin{equation}
\label{eq:L-layer bicriterion minimization problem}
    \min_{f\in \setofnns{L,\width}} 
    \left(\emloss{S}{f}, \Repregbias{L}(f;\width)\right).
\end{equation}

Our goal is to separate between learning rules returning depth-2 Pareto optimal points in $\pareto{2}$ and those returning depth-3 Pareto optional points in $\pareto{3}$.  To make such a rule concrete, one still needs to choose {\em which} Pareto optimal point to return, e.g. choosing a value of $\lambda$ in \eqref{eq:L-layer RERM}.  In order to show separation, we compare the best possible depth-2 rule with a concrete depth-3 rule, showing that a concrete depth-3 rule ``succeeds'', but even the best possible depth-2 rule, and hence any rule returning a depth-2 Pareto optimal point, will ``fail''.

To obtain upper bounds (i.e.,~show learning is easy) we consider the following concrete rule, where the point on the frontier is specified by a threshold $\theta$, as well as its finite-precision relaxations:
\begin{definition}
 \label{def:practical learning rules}
     Given $\theta \ge 0$, define $\realrule{L}$ to be a learning rule which, given training samples $S$, selects an $L$-layer network such that $\emloss{S}{\realrule{L}(S)} \le \theta$ and 
    \begin{equation}
    \label{eq:rep cost is minimial}
        \Repregbias{L}(\realrule{L}(S)) = \inf_{\substack{f \in \gNbar_L \\ \emloss{S}{f} \le \theta}} \Repregbias{L}(f).
    \end{equation}
Given $\alpha \ge 1$, define $\alpharealrule{L}$ to be a learning rule which selects an $L$-layer network such that $\emloss{S}{\alpharealrule{L}(S)} \le \alpha \theta$ and 
    \begin{equation}
    \label{eq:rep cost is alpha-minimial}
        \Repregbias{L}(\alpharealrule{L}(S)) \le \alpha \inf_{\substack{f \in \gNbar_L \\ \emloss{S}{f} \le \theta}} \Repregbias{L}(f).
    \end{equation}
    Similarly, define a bounded width version $\alpharealrule{L,\width}$, to be a learning rule that selects an $L$-layer network of hidden width at most $\width$ such that $\emloss{S}{\alpharealrule{L,\width}(S)} \le \alpha \theta$  and
    \begin{equation}
    \label{eq:bounded width rep cost is alpha-minimial}
        \Repregbias{L}(\alpharealrule{L,\width}(S);\width) \le \alpha \inf_{\substack{f \in \setofnns{L,\width} \\ \emloss{S}{f} \le \theta}} \Repregbias{L}(f;\width).
    \end{equation}
 \end{definition}
 
    The output of $\alpharealrule{L}$ is $\alpha$-close to $\realrule{L}$, which lies on the Pareto frontier. However, we do not require exact Pareto optimality for $\alpharealrule{L}$. See \Cref{fig:learning rules} for a visualization of possible outputs of $\realrule{L}$ and $\alpharealrule{L}$ in relation to the Pareto frontier. 

On the other hand, to prove lower bounds (i.e., argue learning is hard) we consider the following ``ideal'' rule, which ``cheats'' and chooses the Pareto optimal point minimizing the population error, and is thus better than any other rule returning Pareto optimal points:

\begin{definition}
\label{def:idealized learning rule}
    We define $\idealrule{L}$ to be the learning rule which, given training samples $S$, selects the $L$-layer network that minimizes the population loss $\poploss_{\xydist_d}$ over the set $\pareto{L}$ of all Pareto optimal functions for the bicriterion 
    minimization problem in \Cref{eq:general bicriterion minimization problem}. That is, given training samples $S$,
        \begin{equation}
            \idealrule{L}(S) \in \argmin_{f \in \pareto{L}} \poploss_{\xydist_d}(f).
        \end{equation}
    Similarly, we define $\idealrule{L,\width}$ to be the bounded-width version of this idealized rule; 
        \begin{equation}
            \idealrule{L,\width}(S) \in \argmin_{f \in \pareto{L,\width}} \poploss_{\xydist_d}(f).
        \end{equation}
\end{definition}
Strictly speaking, $\idealrule{L}$ is not a learning rule because it depends on knowledge of the true target distribution instead of just samples from that distribution.  It instead can be thought of as an oracle learning rule, based on side knowledge, and thus a lower bound on any learning rule returning Pareto optimal points in $\pareto{L}$.

\begin{remark}
    For $L=2$, \cite{parhi2021banach, unser2023ridges} show the infimum in \Cref{eq:rep cost is minimial} is attained.
    For $L > 2$, it is an open question whether this infimum is attained. If it is not, one can choose a value of $\alpha$ arbitrarily close to 1 and consider $\alpharealrule{L}$ instead of $\realrule{L}$, for which our results still hold.
    For $\alpharealrule{L,\width}(S)$ to exist, we also need $\width$ to be sufficiently large. For example, it suffices that $\width$ is large enough for interpolation of the samples to be possible (see, e.g.,  \cite{yun2019small}).
    It is also possible that the argmins in \Cref{def:idealized learning rule} are not attained. While we state the definition in terms of minimizing the population loss, our results hold even if $\idealrule{L}$ is replaced by any rule that outputs a function on the Pareto frontier.
\end{remark}

    In our main results, we equip Alice with $\idealrule{2}$ to give her the best possible choice of learning with a depth-2 network. 
    However, we allow Bob to use the weaker learning rule $\realrule{3}$ or even $\alpharealrule{3}$.

\begin{figure}[ht!]
    \centering
      \tikzset{every picture/.style={line width=0.75pt}}

\begin{tikzpicture}[x=0.75pt,y=0.75pt,yscale=-1,xscale=1]

\draw  [draw opacity=0][fill={rgb, 255:red, 208; green, 2; blue, 27 }  ,fill opacity=0.33 ][dash pattern={on 0.75pt off 7500pt}] (192.02,162.69) .. controls (170.41,162.69) and (157.03,158.08) .. (142.63,123.53) .. controls (128.22,88.97) and (155.2,40.87) .. (167.2,34.87) .. controls (179.2,28.87) and (200.2,22.87) .. (218.2,22.87) .. controls (236.2,22.87) and (245.2,26.87) .. (266.2,35.87) .. controls (287.2,44.87) and (286.2,50.87) .. (295.2,67.87) .. controls (304.2,84.87) and (306.2,122.87) .. (306.2,132.87) .. controls (306.2,142.87) and (301.7,180.07) .. (294.37,180.07) .. controls (287.04,180.07) and (269.7,169.02) .. (261.98,162.69) .. controls (254.27,156.35) and (242.44,154.63) .. (231.12,156.93) .. controls (219.8,159.23) and (213.63,162.69) .. (192.02,162.69) -- cycle ;
\draw[draw opacity=0.2]    (155.13,147.11) -- (155.13,190) ;
\draw[draw opacity=0.2]   (154.23,147.11) -- (132,147.11) ;
\draw [color={rgb, 255:red, 0; green, 0; blue, 0 }  ,draw opacity=1 ][line width=1.5]  (119.57,180.82) -- (358.57,180.82)(139.57,35.62) -- (139.57,200.67) (351.57,175.82) -- (358.57,180.82) -- (351.57,185.82) (134.57,42.62) -- (139.57,35.62) -- (144.57,42.62)  ;
\draw [color={rgb, 255:red, 65; green, 117; blue, 5 }  ,draw opacity=0.33 ][line width=1.5]  [dash pattern={on 1.69pt off 2.76pt}]  (95.14,66.81) -- (213.46,226.11) ;
\draw [color={rgb, 255:red, 65; green, 117; blue, 5 }  ,draw opacity=1 ][line width=1.5]    (329.57,131.33) -- (345.32,120) ;
\draw [shift={(348.57,117.67)}, rotate = 144.27] [fill={rgb, 255:red, 65; green, 117; blue, 5 }  ,fill opacity=1 ][line width=0.08]  [draw opacity=0] (6.97,-3.35) -- (0,0) -- (6.97,3.35) -- cycle    ;
\draw  [color={rgb, 255:red, 65; green, 117; blue, 5 }  ,draw opacity=1 ][line width=1.5]  (334.24,128.62) -- (337.24,132.67) -- (333.57,135.38) -- (330.57,131.33) -- cycle ;
\draw [color={rgb, 255:red, 144; green, 19; blue, 254 }  ,draw opacity=1 ]   (165.35,124.47) -- (132,124.47) ;
\draw [color={rgb, 255:red, 208; green, 2; blue, 27 }  ,draw opacity=1 ][line width=1.5]    (294.37,180.07) .. controls (284.37,180.07) and (269.7,169.02) .. (261.98,162.69) ;
\draw [color={rgb, 255:red, 208; green, 2; blue, 27 }  ,draw opacity=1 ][line width=1.5]    (192.02,162.69) .. controls (170.41,162.69) and (157.03,158.08) .. (142.63,123.53) .. controls (140.9,119) and (140.9,114.33) .. (140.9,114.33) ;

\draw  [color={rgb, 255:red, 0; green, 0; blue, 0 }  ,draw opacity=1 ][fill={rgb, 255:red, 0; green, 0; blue, 0 }  ,fill opacity=1 ] (270.71,162.95) -- (271.91,167.77) -- (277.19,167.35) -- (272.65,169.9) -- (274.72,174.45) -- (270.71,171.21) -- (266.71,174.45) -- (268.77,169.9) -- (264.24,167.35) -- (269.52,167.77) -- cycle ;
\draw  [color={rgb, 255:red, 208; green, 2; blue, 27 }  ,draw opacity=1 ][fill={rgb, 255:red, 255; green, 255; blue, 255 }  ,fill opacity=1 ][line width=0.75]  (261.01,162.69) .. controls (261.01,163.23) and (261.45,163.66) .. (261.98,163.66) .. controls (262.52,163.66) and (262.96,163.23) .. (262.96,162.69) .. controls (262.96,162.15) and (262.52,161.72) .. (261.98,161.72) .. controls (261.45,161.72) and (261.01,162.15) .. (261.01,162.69) -- cycle ;
\draw [color={rgb, 255:red, 144; green, 19; blue, 254 }  ,draw opacity=1 ]   (165.35,124.47) -- (165.35,190) ;
\draw  [draw opacity=0][fill={rgb, 255:red, 144; green, 19; blue, 254 }  ,fill opacity=0.33 ] (165.35,124.47) -- (165.35,156.56) .. controls (157.35,151.29) and (150.31,141.68) .. (143.03,124.47) -- (165.35,124.47) -- cycle ;
\draw [color={rgb, 255:red, 65; green, 117; blue, 5 }  ,draw opacity=0.33 ][line width=1.5]  [dash pattern={on 1.69pt off 2.76pt}]  (136.14,39.03) -- (254.46,198.33) ;
\draw [color={rgb, 255:red, 65; green, 117; blue, 5 }  ,draw opacity=0.33 ][line width=1.5]  [dash pattern={on 1.69pt off 2.76pt}]  (184.14,15.03) -- (302.46,174.33) ;
\draw [color={rgb, 255:red, 65; green, 117; blue, 5 }  ,draw opacity=0.33 ][line width=1.5]  [dash pattern={on 1.69pt off 2.76pt}]  (231.14,-0.97) -- (349.46,158.33) ;
\draw  [color={rgb, 255:red, 0; green, 0; blue, 0 }  ,draw opacity=1 ][fill={rgb, 255:red, 0; green, 0; blue, 0 }  ,fill opacity=1 ][line width=0.75]  (152.24,147.11) .. controls (152.24,148.71) and (153.53,150) .. (155.13,150) .. controls (156.72,150) and (158.02,148.71) .. (158.02,147.11) .. controls (158.02,145.52) and (156.72,144.22) .. (155.13,144.22) .. controls (153.53,144.22) and (152.24,145.52) .. (152.24,147.11) -- cycle ;
\draw [color={rgb, 255:red, 144; green, 19; blue, 254 }  ,draw opacity=1 ]   (171.89,118.81) .. controls (178.88,123.3) and (181.88,128.8) .. (179.88,131.8) .. controls (178.15,134.39) and (167.44,136.05) .. (162.31,136.78) ;
\draw [shift={(160.38,137.05)}, rotate = 351.87] [color={rgb, 255:red, 144; green, 19; blue, 254 }  ,draw opacity=1 ][line width=0.75]    (10.93,-3.29) .. controls (6.95,-1.4) and (3.31,-0.3) .. (0,0) .. controls (3.31,0.3) and (6.95,1.4) .. (10.93,3.29)   ;

\draw (115,14.4) node [anchor=north west][inner sep=0.75pt]    {$\Repregbias{L}(f)$};
\draw (363.9,171.73) node [anchor=north west][inner sep=0.75pt]    {$\emloss{S}{f}$};
\draw (149,190) node [anchor=north west][inner sep=0.75pt]  [font=\normalsize]  {$\theta$};
\draw (158,190) node [anchor=north west][inner sep=0.75pt]  [font=\normalsize,color={rgb, 255:red, 144; green, 19; blue, 254 }  ,opacity=1]  {$\alpha \theta$};
\draw (133,141) node [anchor=north east][inner sep=0.75pt]    {$\inf\limits_{\emloss{S}{f} \le \theta} \Repregbias{L}(f)$};
\draw (133,114) node [anchor=north east][inner sep=0.75pt]  [color={rgb, 255:red, 144; green, 19; blue, 254 }  ,opacity=1 ]  {$\alpha  \inf\limits_{\emloss{S}{f} \le \theta} \Repregbias{L}(f)$};
\draw (265,145) node [anchor=north west][inner sep=0.75pt]  [color={rgb, 255:red, 0; green, 0; blue, 0 }  ,opacity=1 ]  {$\idealrule{L}$};
\draw (145,95) node [anchor=north west][inner sep=0.75pt]  [color={rgb, 255:red, 144; green, 19; blue, 254 }  ,opacity=1 ]  {$\alpharealrule{L}$};
\draw (140,151.5) node [anchor=north west][inner sep=0.75pt]  [color={rgb, 255:red, 0; green, 0; blue, 0 }  ,opacity=1 ]  {$\realrule{L}$};
\draw (350.47,112.45) node [anchor=north west][inner sep=0.75pt]  [font=\tiny,color={rgb, 255:red, 65; green, 117; blue, 5 }  ,opacity=1 ]  {$\begin{bmatrix}
1\\
\lambda 
\end{bmatrix}$};

\end{tikzpicture}
      \caption{ 
         \emph{Visualization of $\realrule{L}(S)$,  $\alpharealrule{L}(S)$, and $\idealrule{L}(S)$.}
          The red shaded area represents the set of possible values of $\left(\emloss{S}{f}, \Repregbias{L}(f)\right)$ where $f$ is represented by an $L$-layer network. 
          The red curves form the Pareto frontier $\pareto{L}$.
          Minimizing the population loss $\poploss_{\xydist_d}$ over the Pareto frontier yields $\idealrule{L}(S)$, represented by the star. 
          In green is the vector $[1,\lambda]^\top$ and lines normal to it. These normal lines form level sets of $\emloss{S}{f} + \lambda \Repregbias{L}(f)$. 
          Notice the black dot on the Pareto frontier, which represents $\realrule{L}(S)$. The output of $\realrule{L}(S)$ corresponds to $\min_{f\in \setofnns{L}} \emloss{S}{f} + \lambda \Repregbias{L}(f).$
          The purple shaded region shows the possible outputs of $\alpharealrule{L}(S)$, 
          which are all $\alpha$-close to
          $\realrule{L}(S)$.
          \label{fig:learning rules}}
\end{figure}
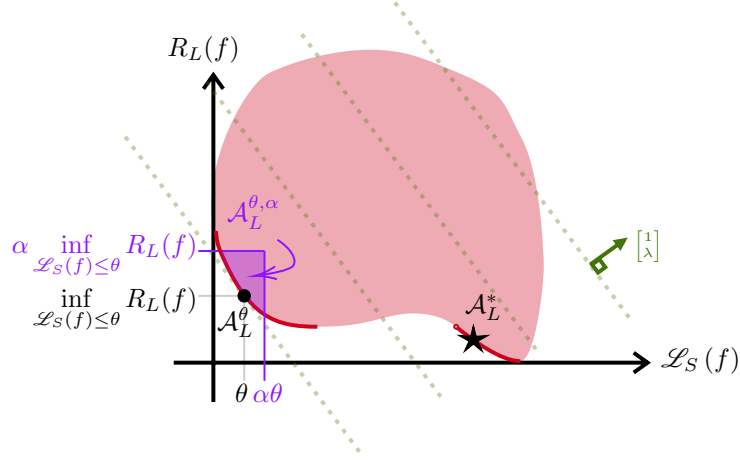

\section{Main Results: Norm-Based Depth Separation in Learning}
\label{sec:main results}
We now state our two main theorems. 
\Cref{thm:depth separation} says that there is a family of functions that $\realrule{3}$ (i.e., Bob) can learn with sample complexity that is polynomial in $d$ but $\idealrule{2}$ (i.e., Alice) needs the number of samples to grow exponentially with $d$ in order to learn.
\Cref{thm:no reverse depth separation} ensures that the reverse does not occur; families of distributions that $\idealrule{2}$ can learn with polynomial sample complexity can also be learned with polynomial sample complexity using $\realrule{3}$. Both results still hold even when we relax Bob's depth-3 learning rule from $\realrule{3}$ to $\alpharealrule{3}$. For ease of presentation, we consider $\alpha$ to be a small constant, e.g., $\alpha = 2$.
\begin{theorem}[Depth Separation in Learning]
\label{thm:depth separation}
    There is a family of distributions $(\xydist_d)_{d=2}^\infty$ on $\gX_d \times [-1,1]$ defined as
    $
        \vx \sim \Unif(\gX_d)$ and 
    $
        y|\vx = f_d(\vx)
    $
    for some function $f_d: \gX_d \rightarrow [-1,1]$
    such that the following holds.
    \begin{enumerate}
        \item There are real numbers $d_0 >0$ and $C_1 > 0$, such that 
        if $d > d_0$ and  
        $|S| < 2^{C_1d}$, then 
        $
             \E_{S}[\poploss_{\xydist_d}(\idealrule{2}(S))]
            \ge 0.0001
        $.
        
        \item 
        For all $\varepsilon,\delta > 0$, if
        $
            \theta = \frac{\varepsilon}{4}
        $ and 
        $
            |S| > O\left(\frac{d^{15} \log(1/\delta)}{\varepsilon^2}\right),
        $
        then
        $
            \poploss_{\xydist_d}(\realrule{3}(S)) \le \varepsilon
        $
        with probability at least $1-\delta$.
        Furthermore, with a fixed constant $\alpha \ge 1$,
        $
            \poploss_{\xydist_d}(\alpharealrule{3}(S)) \le \varepsilon
        $ 
        with probability at least $1-\delta$ where now the big-$O$ suppresses a constant that depends on $\alpha$.
    \end{enumerate}
\end{theorem}
\begin{theorem}[No Reverse Depth Separation in Learning]
\label{thm:no reverse depth separation}
    Consider a distribution $\xydist_d$ on $\gX_d \times [-1,1]$
    defined as
    $
        \vx \sim \Unif(\gX_d)
    $
    and 
    $
        y|\vx = f_d(\vx)
    $
    for some function $f_d: \gX_d \rightarrow [-1,1]$.
    Assume that there is 
    some sample complexity function $m_2(\varepsilon)$ such that $\E_{S}[\poploss_{\xydist_d}(\idealrule{2}(S))] \le \varepsilon$
    whenever $|S| \ge m_2(\varepsilon)$.
    
    For all $\varepsilon,\delta > 0$, if $\theta = \frac{\varepsilon}{4}$ and $|S| \ge m_3(\varepsilon,\delta)$,
    then $\poploss_{\xydist_d}(\realrule{3}(S)) \le \varepsilon$  with probability at least $1-\delta$,
    where the sample complexity $m_3$ is
    \begin{equation}
        m_3(\varepsilon,\delta) = 
        O\left(\varepsilon^{-2} \left(d + m_2\left(\frac{\varepsilon}{64}\right)^{\frac{d+3}{d-1}} \right)^6 \log{1/\delta}\right).
    \end{equation}
    Furthermore, with a fixed constant $\alpha \ge 1$, 
    $\poploss_{\xydist_d}(\alpharealrule{3}(S)) \le \varepsilon$  with probability at least $1-\delta$ where now the big-$O$ suppresses a constant that depends on $\alpha$.
    
    In particular, if we have a family of such distributions $(\xydist_d)_{d=2}^\infty$ and $m_2$ grows polynomially with $d$, then $m_3$ also grows polynomially with $d$.
\end{theorem}
\begin{remark}
    \Cref{thm:depth separation,thm:no reverse depth separation} are based on loose bounds. We conjecture that smaller sample complexities for depth-3 learning are possible in both results. Additionally, larger lower bounds on generalization for depth-2 learning are possible in \Cref{thm:depth separation}. 
\end{remark}

\begin{remark}
\label{remark:bounded width results}
    We can generalize these results to networks of bounded widths. In
    \Cref{thm:depth separation}, Part 1 holds for $\idealrule{2,\width}$ as long as the width is at least three times the sample size, i.e., $\width > 3|S|$. Thus, if the sample size is polynomial in $d$, then in sufficiently high dimensions, $\idealrule{2,\width}$ cannot generalize without width that is super-polynomial in $d$.
    Part 2 holds for $\alpharealrule{3,\width}$ as long as $\width \ge O\left(\varepsilon^{-1/2} d^{7/2}\right)$. That is, for depth-3 learning, we only require a width that is polynomial in dimension.
    To generalize
    \Cref{thm:no reverse depth separation} to bounded-width networks, we can modify the premise to the assumption that there is some minimal width function $\minwidth(\varepsilon)$ such that $\E_{S}[\poploss_{\xydist_d}(\idealrule{2,\width}(S))] \le \varepsilon$
    whenever $|S| \ge m_2(\varepsilon)$ and $\width \ge \minwidth(\varepsilon)$. The width $\width$ required for $\alpharealrule{3,\width}$ to learn is then
    $
        \width \ge O\left(\varepsilon^{-1} m_2\left(\frac{\varepsilon}{64}\right)^{\frac{2(d+3)}{d-1}} + d\right).
    $
    If $m_2$ grows polynomially with $d$, then the width required for depth-3 learning is only polynomial in $d$.
\end{remark}

\begin{remark}
    The relatively restrictive assumptions on the distribution of $\vx$ in \Cref{thm:no reverse depth separation} can be relaxed. We use these assumptions to bound the $\Repregbias{2}$ cost of interpolating samples. Our particular construction would be straightforward to generalize to other smooth distributions on $\gX_d$ or $\bbS^{d-1}$. Other constructions could yield bounds on the $\Repregbias{2}$ cost of interpolating samples from other smooth distributions, which would allow for generalizations of this result.
\end{remark}

\section{Proof of Depth Separation in Learning}
\label{sec:dep sep}

For the proof of \Cref{thm:depth separation}, we use a slight modification of the construction from \cite{daniely2017depth}. 
We choose $f_d(\vx) := \psi_{3d}\left(\sqrt{d} \langle \vx^{(1)}, \vx^{(2)} \rangle\right)$
where $\psi_{n}: \R \rightarrow [-1,1]$ denotes the sawtooth function that has $n$ cycles in $[-1,1]$ and is equal to zero outside $[-1,1]$.  
See \Cref{fig:sawtooth} for a depiction of $\psi_{n}$.
This target function is convenient for studying depth separation in norm because the sawtooth function can be represented exactly with one hidden ReLU layer, while the inner product can be approximated with another hidden ReLU layer. Thus, $f_d$ lends itself well to explicit bounds on the $R_3$ representation cost needed to approximate it. Since $f_d$ is a composition with an inner product, the framework in \cite{daniely2017depth} allows us to get a bound on the $R_2$ representation cost needed to approximate it as well. See \Cref{lem:need large R2 to approx in L2 norm,lem:approximation of target with small error and cost}. 

As with other depth-separation constructions, the Lipschitz constant of $f_d$ is unbounded as $d$ goes to infinity.  Obtaining depth separation as $d$ goes to infinity but with a bounded Lipschitz constant is a yet unsolved challenge; see \cite{safran2019depth} for a discussion and evidence that current techniques cannot be used to show depth separation with a bounded Lipschitz constant.

In the following two subsections, we sketch the proofs of Parts 1 and 2 of \Cref{thm:depth separation}.
    
\subsection{Proof of \Cref{thm:depth separation}, Part 1}
\label{sec:cant learn w two layers}
\begin{proof}
    Using \Cref{cor:depth separation in norm to approximate}, the construction of \cite{daniely2017depth} requires the $R_2$  cost to grow exponentially in $d$ to approximate the target in the $L_2$ norm. We adapt this construction to $f_d$ in \Cref{lem:need large R2 to approx in L2 norm} to get more explicit bounds, from which
    we conclude that there exist real numbers $d_0,C > 0$ such that $d > d_0$ and $\Repregbias{2}(f) < 2^{Cd}$ implies that $\poploss_{\xydist_d}(f) \ge \frac{1}{50e^2 \pi^2}$.
    
    If $d \ge 3$ then by \Cref{lem:tent interpolator is cheap}, 
    with probability at least $\frac{1}{2}$ there exists an interpolant $\hat f \in \setofnns{2}$ of the samples $S$ with representation cost bounded as $\Repregbias{2}(\hat f) \le 32 \sqrt{2}|S|^{\frac{d+3}{d-1}}$. The proof of \Cref{lem:tent interpolator is cheap} relies on the fact that with high probability the samples are sufficiently separated, and separated samples on $\gX_d$ can be interpolated by a depth-2 neural network with small norm parameters.
    Similar ways to construct interpolants exist in other settings; see for example Section 5.2 in \cite{ongie2019function}.
    Since $\idealrule{2}(S) \in \pareto{2}$ is Pareto optimal, we must have that
    $
        \Repregbias{2}(\idealrule{2}(S)) \le  \Repregbias{2}(\hat f).
    $
    Otherwise, $\idealrule{2}(S)$ would fail to be Pareto optimal because $\hat f$ would have a smaller sample loss and a smaller representation cost.
    It follows that $\Repregbias{2}(\idealrule{2}(S)) \le 32 \sqrt{2}|S|^{\frac{d+3}{d-1}}$ with probability at least $\frac{1}{2}$. 

    Choose $C_1 = C/4$. Assume $d > \max(d_0,3,\frac{11}{2C_1})$ and $|S| < 2^{C_1d}$. 
    With probability at least $\frac{1}{2}$, we must have 
    \begin{equation}
        \Repregbias{2}(\idealrule{2}(S))
        \le 32 \sqrt{2}|S|^{\frac{d+3}{d-1}}
        < 2^{C_1 d} 2^{\frac{C_1 d(d+3)}{d-1}}
        \le 2^{C d}.
    \end{equation}
    Thus, $\poploss_{\xydist_d}(\idealrule{2}(S)) \ge \frac{1}{50e^2 \pi^2}$
    with probability at least $\frac{1}{2}$. Therefore, by Markov's inequality, 
    $
         \E_{S}[\poploss_{\xydist_d}(\idealrule{2}(S))]
         \ge \frac{1}{50e^2 \pi^2} \cdot \frac{1}{2}
        \ge 10^{-4}
    $
\end{proof}
\subsection{Proof of \Cref{thm:depth separation}, Part 2}
\label{sec:real rule can learn}
We prove a slightly more general version of Part 2 in \Cref{thm:depth separation}. Instead of just proving the result for $\realrule{3}$ or $\alpharealrule{3}$, we prove the result for the relaxed, bounded width learning rule $\alpharealrule{3,\width}$ for any $\alpha \ge 1$. This illuminates how the sample complexity and width we require to guarantee learning depends on $\alpha$ and $\width$.
\begin{proof}
    Fix $\varepsilon,\delta > 0$ and $\alpha \ge 1$,  and let $\theta = \frac{\varepsilon}{2\alpha }$.    
    In \Cref{lem:approximation of target with small error and cost} we show that for all $K \in \N$
    there is a depth-3 neural network $f_{d,K}$ of width $\width_{d,K} := \max(6d+2,2Kd)$ such that 
    $
        \|f_d - f_{d,K}\|_{L^\infty}
        = O\left( \frac{d^{5/2}}{K} \right)
    $
    and 
    $
        \Repregbias{3}(f_{d,K};\width_{d,K}) = O(d^{5/2}).
    $
    Hence 
    $
        \emloss{S}{f_{d,K}} 
        = O\left(\frac{d^{5}}{K^2}\right).
    $
    We choose $K \ge O\left(\frac{d^{5/2}}{\sqrt{\theta}}\right)$
    so that $\emloss{S}{f_{d,K}} \le \theta$. Now suppose that $\width \ge \width_{d,K}$.
    Then 
    \begin{align}
        \Repregbias{3}(\alpharealrule{3,\width}(S)) 
        &\le \Repregbias{3}(\alpharealrule{3,\width}(S);\width) 
        \le \alpha  \inf_{\substack{g \in \setofnns{3,\width} \\ \emloss{S}{g} \le \theta}} \Repregbias{3}(g;\width) \\
        &\le \alpha  \Repregbias{3}(f_{d,K};\width) 
        = O(\alpha d^{5/2}).
    \end{align}

    In \Cref{lem:estimation error bound} we use the Rademacher complexity bounds from \cite{neyshabur2015norm} to get the following estimation error bound on $f \in \setofnns{3}$ with respect to the target distribution $\xydist_d$; if $\Repregbias{3}(f) \le M$, then 
    $
        |\poploss_{\xydist_d}(f) - \emloss{S}{f}| \le O\left(M^3\sqrt{\frac{ \log{1/\delta}}{|S|}}\right)
    $
    with probability at least $1-\delta$.
    Applying this, we get
    \begin{align}
        \poploss_{\xydist_d}(\alpharealrule{3,\width}(S)) 
        &\le \emloss{S}{\alpharealrule{3,\width}(S)} + |\poploss_{\xydist_d}(\alpharealrule{3,\width}(S)) - \emloss{S}{\alpharealrule{3,\width}(S)}| \\
        &\le \alpha \theta + O\left(\sqrt{\frac{\alpha ^{6}d^{15}\log{1/\delta}}{|S|}}\right)
    \end{align}
    with probability at least $1-\delta$.
    Therefore, with 
    \begin{equation}
        |S| > O\left(\frac{\alpha ^6d^{15}\log{1/\delta}}{\varepsilon^2}\right)
        \text{~and~}
        \width \ge \width_{d,K} 
        = O\left(
            \sqrt{\frac{\alpha d^7}{\varepsilon}}
        \right),
    \end{equation}
    we get 
    $
        \poploss_{\xydist_d}(\alpharealrule{3,\width}(S)) 
        \le \alpha \theta + \frac{\varepsilon}{2} = \varepsilon
    $    
    with probability at least $1-\delta$.
\end{proof}
\section{Proof of No Reverse Depth Separation in Learning}
\label{sec:no rev dep sep proof}

To prove \Cref{thm:no reverse depth separation}, we need the following lemma. Roughly speaking, this lemma says that if $\idealrule{2}(S)$ can learn with $m_2$ samples, then there is a good approximation of the target distribution that can be expressed as a depth-2 network with parameters whose norm is at most polynomial in $m_2$.
The proof is a straightforward probabilistic argument, shown in \Cref{sec:proof learning with two layers means approximate with small cost}. 
\begin{lemma}
\label{lem:learning with two layers means approximate with small cost}
    Consider a distribution $\xydist_d$ on $\gX_d \times [-1,1]$
    defined as
    $
        \vx \sim \Unif(\gX_d)
    $
    and
    $
        y|\vx = f_d(\vx)
    $
    for some function $f_d: \gX_d \rightarrow [-1,1]$.
    Assume that there is 
    some sample complexity function $m_2(\varepsilon)$ such that $\E_{S}[\poploss_{\xydist_d}(\idealrule{2}(S))] \le \varepsilon$
    whenever $|S| \ge m_2(\varepsilon)$.
    Then for any $\varepsilon > 0$, 
    there is a function $f_{\varepsilon} \in \setofnns{2}$ such that 
    $
        \Repregbias{2}(f_{\varepsilon}) 
        \le 100 \sqrt{2}m_2\left(\frac{\varepsilon}{2}\right)^{\frac{d+3}{d-1}}
    $
    and $\poploss_{\xydist_d}(f_{\varepsilon}) \le \varepsilon$. 
\end{lemma}
Using the previous lemma and \Cref{lem:can approximate with narrow network with same R2 cost}, the rest of the proof of \Cref{thm:no reverse depth separation} follows from the estimation error bound in \Cref{lem:estimation error bound} derived from Rademacher complexity bounds and the fact that any function with small $\Repregbias{2}$-cost also has small $\Repregbias{3}$-cost. This fact is shown in \Cref{lem:ub on R3b by R2b with identity layer} by adding an identity layer.
As in \Cref{sec:real rule can learn}, we prove \Cref{thm:no reverse depth separation} for the relaxed, bounded width learning rule $\alpharealrule{3,\width}$ for any $\alpha \ge 1$ to showcase the role of $\alpha$ and $\width$, but this proof also applies to $\alpharealrule{3}$ and $\realrule{3}$. Full details are in \Cref{sec:proof details no rev dep sep}.

\section{Conclusion}
\label{sec:conclusion}
This paper demonstrates that there are functions that can be learned with depth-3 networks when the number of samples is polynomial in the input dimension $d$, but which cannot be learned with depth-2 networks unless the number of samples is exponential in $d$. Furthermore, we establish that in our setting, there are \textit{no} functions that can easily be learned with depth-2 networks but which are difficult to learn with depth-3 networks. These results constitute the first depth separation result in terms of \textit{learnability}, as opposed to network width. 

In addition, the analysis framework we develop in this paper establishes a connection between width-based depth separation and learnability-based depth separation. As a result, our approach may be applied to other works on width-based depth separation to establish new learnability-based depth separation results. 

We note that while the bounds developed in this paper are sufficient to establish our main results on depth separation, they may not be tight. For instance, the sample complexity bounds for depth-3 networks grow polynomially in $d$, but the polynomial order is quite large. Alternative constructions might lead to tighter bounds. Furthermore, the family of functions we use to establish our depth separation results does not have bounded Lipschitz constants; as $d$ grows, our functions become highly oscillatory. Since highly oscillatory functions may not be representative of many practical predictors, it would be interesting to see whether there are families of functions \textit{with bounded Lipschitz constants} leading to depth separation in terms of sample complexity (we note that \citep{safran2017depth} studied this question but in the different context of width). A final potential limitation of our work is that it focuses on the output of learning rules seeking (approximately) Pareto optimal solutions, but neglects optimization dynamics. A major open question is how optimization dynamics affect depth separation. 

\section{Acknowledgements}
% Becca: AFOSR FA9550-18-1-0166 and NSF DMS-2023109, 

% Nati: Supported in part by the NSF/TRIPOD funded Institute of Data, Econometrics, Algorithms and Learning (IDEAL), by the NSF/Simons funded Collaboration on the Theoretical Foundations of Deep Learning. and by an NSF/IIS award.}
S.P. was supported by the NSF Graduate Research Fellowship Program under Grant No. 2140001. Any opinions, findings, and conclusions or recommendations expressed in this material are those of the authors and do not necessarily reflect the views of the NSF.
G.O. was supported by NSF CRII award CCF-2153371.
R.W. and S.P. were supported by NSF DMS-2023109; R.W. was also supported by AFOSR FA9550-18-1-0166. 
O.S. was supported in part by European Research Council (ERC) grant 754705.
N.S. was supported by the NSF/TRIPOD funded Institute of Data, Econometrics, Algorithms, and Learning (IDEAL), by the NSF/Simons funded Collaboration on the Theoretical Foundations of Deep Learning, and by NSF/IIS award 1764032.

\newpage
\bibliography{refs}
\newpage
\appendix
\onecolumn

\section{Technical Lemmas \& Proofs}
\label{sec:technical proofs}

Here, we present the technical details of the results in the main text.

\subsection{Characterizing and bounding the representation cost}
\label{sec:rep cost bounds}

In this section, characterizations of and bounds on the representation cost that we use elsewhere in the appendix.
To ease notation in this section, we re-label parameters defining a depth-2 network $f_{\bm\phi}$ as $\bm\phi = (\mW,\vb,\va,c)$ so that $f_{\bm\phi}(\vx) = \sum_{k=1}^{\width_1} a_k[\vw_k^\top\vx + b_k]_+ + c$, where $\vw_k^\top$ is the $k$th row of $\mW$ and $\width_1$ is the width of the hidden layer in the parameterization.

The first result shows that the depth-2 representation cost reduces to the $\ell^1$-norm of the outer-layer weights (plus half the squared outer-layer bias) assuming the first-layer weights/biases are normalized:
\begin{lemma}
\label{lem:reduced depth-2 rep cost} 
Let $f \in \setofnns{2}$. Then
\begin{align}
R_2(f) 
&= \inf_{\bm\phi \in \Phi_2 : f = f_{\bm\phi}} \sum_{k=1}^{\width_1} |a_k| + \frac{c^2}{2}~~s.t.~~\|\vw_k\|^2 + |b_k|^2 = 1~\forall k \in [\width_1] \\
 &= \inf_{\bm\phi \in \Phi_2 : f = f_{\bm\phi}} \sum_{k=1}^{\width_1} |a_k|\sqrt{\|\vw_k\|^2 + |b_k|^2} + \frac{c^2}{2}.
\end{align}
Similarly, given $f \in \setofnns{2,\width}$, we have a bounded width version of this:
\begin{align}
R_2(f;\width) 
&= \inf_{\substack{\bm\phi \in \Phi_2 : f = f_{\bm\phi} \\ \width_1 \le \width }} \sum_{k=1}^{\width_1} |a_k| + \frac{c^2}{2}~~s.t.~~\|\vw_k\|^2 + |b_k|^2 = 1~\forall k \in [\width_1] \\
 &= \inf_{\substack{\bm\phi \in \Phi_2 : f = f_{\bm\phi} \\ \width_1 \le \width }} \sum_{k=1}^{\width_1} |a_k|\sqrt{\|\vw_k\|^2 + |b_k|^2} + \frac{c^2}{2}.
\end{align}
\end{lemma}
We omit a full proof for brevity, but the result is a trivial modification of Lemma 1 in Appendix A of \cite{savarese2019infinite}, extended to the case of regularized bias terms considered in this work.
See also \cite{boursier2023penalising}. 

The next result says that functions that have small representation costs with depth-2 networks also have small representation costs with depth-3 networks. The proof adds an identity layer to a depth-2 network to turn it into a depth-3 network.
\begin{lemma}
\label{lem:ub on R3b by R2b with identity layer}
    Given $f \in\setofnns{2,\width}$, we have $f \in\setofnns{3,\max(\width,4d)}$ and
    \begin{equation}
        \Repregbias{3}(f;\max(\width,4d)) \le \frac{4d}{3} + \frac{4}{3} \Repregbias{2}(f;\width).
    \end{equation}
\end{lemma}
\begin{proof}
    Assume that $f \in\setofnns{2,\width}$. 
    Fix a particular parameterization $\bm\phi = (\mW,\vb,\va,c)$ of $f$ of width $\width$.
    Since $[\vx]_+ - [-\vx]_+ = \vx$, we can rewrite $f$ as a depth-3 neural network with an identity layer:
    \begin{align}
        f(\vx)
        = f_{\bm\phi}(\vx)
        &= \va^\top \left[\mW\vx  + \vb \right]_+ + c \\
        &= \va^\top \left[ \begin{bmatrix}\mW & -\mW\end{bmatrix} \left[\begin{bmatrix}\mI_{2d} \\ -\mI_{2d}\end{bmatrix} \vx\right]_+ + \vb \right]_+ + c \\
        &= f_{\bm\phi'}(\vx),
    \end{align}
    where $\bm\phi'$ is this new parameterization:
    \begin{equation}
        \bm\phi' = \left( 
        \begin{bmatrix}\mI_{2d} \\ -\mI_{2d}\end{bmatrix},\vzero,
        \begin{bmatrix}\mW & -\mW\end{bmatrix},\vb,
        \va,c\right).
    \end{equation}
    Notice that $\bm\phi'$ has one hidden layer of width $4d$ and one hidden layer of width $\width$, so 
    $f \in \setofnns{3,\max(\width,4d)}$.
    Further, 
    \begin{align}
            \|\bm\phi'\|^2
            &= 
                \left\|\begin{bmatrix}\mI_{2d} \\ -\mI_{2d}\end{bmatrix}\right\|_F^2
                + \left\|\begin{bmatrix}\mW & -\mW\end{bmatrix}\right\|_F^2 + \left\|\vb\right\|_2^2 
                + \left\|\va\right\|_2^2 + |c|^2
            \\
            &=
                4d
                + 2\left\|\mW\right\|_F^2 + \left\|\vb\right\|_2^2 
                + \left\|\va\right\|_2^2 + |c|^2
            \\
            &\le 4d + 2 \|\bm\phi\|^2.
    \end{align}
    Therefore,
    \begin{align}
        \Repregbias{3}(f;\max(\width,4d))
        &= \inf_{\bm\phi \in \Phi_3 :  f = f_{\bm \phi}} \frac{\|\bm\phi\|^2}{3} \\
        &\le \inf_{\bm\phi \in \Phi_2 :  f = f_{\bm \phi}} \frac{4d + 2 \|\bm\phi\|^2}{3} \\
        &= \frac{4d}{3} + \frac{4}{3} \Repregbias{2}(f;\width).
    \end{align}
\end{proof}

\subsection{Approximating wide depth-2 networks by narrow networks with the same representation cost}

\label{sec:can approximate with narrow network with same R2 cost}
\subsubsection{Proof of \Cref{lem:can approximate with narrow network with same R2 cost}}
Before proving \Cref{lem:can approximate with narrow network with same R2 cost} we give an auxiliary result needed for the proof.  
The following is a simplified version of Lemma 1 from \cite{barron1993universal}, originally credited to Maurey:
\begin{lemma}[Maurey's Lemma]\label{lem:Maurey}
Let $H$ be a Hilbert space with norm $\|\cdot\|_H$. Assume $\gG \subset H$ is such that $\|g\|_H \leq B$ for all $g \in \gG$. Suppose $f$ is a non-zero function belonging to the closed convex hull of $\gG$. Then for any $m \in \mathbb{N}$ and there exists elements $g_1,...,g_m \in \gG$ such that
\[
\left\|f-\frac{1}{m}\sum_{k=1}^m g_k\right\|_H \leq \frac{B}{\sqrt{m}}.
\]
\end{lemma}

We specialize this result to the Hilbert space $H = L^2(\gX_d)$, and the subset $\gG \subset L^2(\gX_d)$ of all functions consisting of a single normalized ReLU unit. In particular, for any $\vw \in \R^{2d}$ and $b \in \R$, define $u_{\vw,b}(\vx) = [\vw^\top\vx + b]_+$ and let $\gG \subset L^2(\gX_d)$ be the set of functions
\[
\gG = \{\pm u_{\vw,b} : \vw \in \R^{2d}, b\in \R, \|\vw\|^2 + |b|^2 = 1\}.
\]
Let $\mu_d$ denote the uniform probability measure on $\gX_d = \mathbb{S}^{d-1}\times \mathbb{S}^{d-1}$. Note that for any $g = \pm u_{\vw,b} \in \gG$ we have
\begin{align}
\|g\|_{L^2}^2 & = \int_{\gX_d} [\vw^\top\vx+b]_+^2 d\mu_d(\vx)\\
& \leq  \int_{\gX_d} |\vw^\top\vx+b|^2 d\mu_d(\vx)\\
& = \int_{\gX_d} \left|\begin{bmatrix} \vw \\ b\end{bmatrix}^\top \begin{bmatrix} \vx \\ 1\end{bmatrix}\right|^2 d\mu_d(\vx)\\
& \leq \int_{\gX_d} (\|\vw\|^2 + |b|^2)(1 + \|\vx\|^2)d\mu_d(\vx)\\
& = 3\int_{\gX_d}d\mu_d(\vx) = 3,
\end{align} 
where we used the fact that $\|\vx\|^2 = 2$ for all $\vx \in \gX_d$.
Therefore, for $B = \sqrt{3}$ we have $\|g\|_{L^2} \leq B$ for all $g \in \gG$.

Now we give the proof of \Cref{lem:can approximate with narrow network with same R2 cost}:
\begin{proof}
Let $f\in \setofnns{2}$ and $\epsilon > 0$ be given, and suppose $\width \in \mathbb{N}$ is such that $\width > 3 R_2(f)^2/\epsilon^2$. Choose $\delta$ with $0 < \delta \leq 1$ to be any constant satisfying $\width \geq (1+\delta)^2 3 R_2(f)^2/\epsilon^2$, and let $f(\vx) = \sum_{k=1}^K a_k[\vw_k^\top\vx +b_k]_+ + c$ with $\|\vw_k\|^2 + |b_k|^2 =1$ for all $k \in [K]$ be any realization of $f$ whose parameter cost is within a factor of $(1+\delta)$ of the infimum in \Cref{lem:reduced depth-2 rep cost}, i.e., $(1+\delta)R_2(f) \geq \sum_{k=1}^K |a_k| + \frac{c^2}{2}$. Let $A = \sum_{k=1}^K |a_k|$, and define $f_0 = (f-c)/A$. Then we can write $f_0(\vx) = \sum_k \gamma_k s_k [\vw_k^\top\vx +b_k]_+$ where $s_k = \text{sign}(a_k)$ and $\gamma_k = |a_k|/A$ for all $k$. This shows $f_0$ is in the convex hull of $\gG$, since $f_0 = \sum_k \gamma_k g_k$ with $g_k = s_k \, u_{\vw_k,b_k} \in \gG$ and $\gamma_k \geq 0$,  $\sum_k \gamma_k = 1$. 

Therefore, by \Cref{lem:Maurey}, there exists a function $\tilde{f}_0$ of the form $\tilde{f}_0(\vx)= \frac{1}{\width} \sum_{k=1}^\width \tilde{s}_k [\tilde{\vw}_k^\top\vx + \tilde{b}_k]_+$ where $\|\tilde{\vw}_k\|^2 + |\tilde{b}_k|^2=1$ and $\tilde{s}_k \in \{-1,1\}$, such that
\[
\|f_0-\tilde{f}_0\|_{L^2} \leq \frac{\sqrt{3}}{\sqrt{\width}} \leq \frac{\epsilon}{(1+\delta)R_2(f)}.
\]
Multiplying both sides above by $A$ gives
\[
\|(f-c)-A\tilde{f}_0\|_{L^2} \leq \frac{A\epsilon}{(1+\delta)R_2(f)} \leq \frac{A\epsilon}{A+\frac{c^2}{2}} \leq \epsilon.
\]
Defining $\tilde{f} = A\tilde{f}_0 + c$, we have 
\[
\|f -\tilde{f}\|_{L^2} \leq \epsilon,
\]
where $\tilde{f}(\vx) = \frac{1}{\width}\sum_{k=1}^\width s_k A [\tilde{\vw}_k^\top\vx + \tilde{b}_k]_+ + c$ is realizable as a depth-two ReLU network with width at most $\width$.
In particular, with the choice of weights $\bm\phi = (\hat{\mW},\hat{\vb},\hat{\va},c)$ with $\hat{\vw}_k := \sqrt{A/\width}\, \tilde{\vw}_k$, $\hat{b}_k := b_k\sqrt{A/\width}$, $\hat{a}_k := s_k\,\sqrt{A/\omega}$, for all $k\in[\width]$, we have $\tilde{f} = f_{\bm \phi}$ with $\frac{\|\bm\phi\|_2^2}{2} = A + \frac{c^2}{2}\leq (1+\delta)R_2(f) \leq 2 R_2(f)$, and so $\|\bm\phi\|_2^2 \leq 4R_2(f)$. Finally, the inequality $\|\bm\phi\|_\infty^2\leq \|\bm\phi\|_2^2$ holds for any vector $\bm\phi$, which proves the claim.
\end{proof}

\subsubsection{Proof of \Cref{cor:depth separation in norm to approximate}}

\begin{proof} Let $f_d$ be the family of functions described in \Cref{lem:Daniely}. First, to prove the depth-three result, set $\tilde{f}_d$ to be equal to the approximating function $\tilde{f}_{\bm\phi} \in \gN_3$ described in \Cref{lem:Daniely}. By a simple parameter count and the bounds on the magnitudes of weights, we are guaranteed that $R_3(\tilde{f}_{\bm\phi}) \leq \frac{\|\bm\phi\|^2}{3} = O(\text{poly}(d)/\varepsilon)$.

Now, we prove the depth-two result. Set $\varepsilon = 10^{-4}$. By way of contradiction, assume $f_d$ can be $\varepsilon/2$-approximated in $L^2$-norm by a depth-two network $\hat{f}_d$ such that $R_2(\hat{f}_d)$ is subexponential in $d$. Then \Cref{lem:can approximate with narrow network with same R2 cost} implies $\hat{f}_d$ can be $\varepsilon/2$-approximated by a depth-two network $\tilde{f}_d \in \gN_2$ with $R_2(\tilde{f}_d)$ subexponential in $d$, width $\width$ subexponential in $d$, and weights uniformly bounded by $2^d$ for sufficiently large $d$. Hence, by the triangle inequality, $f_d$ can be $\varepsilon$-approximated in $L^2$-norm by the depth-two network $\tilde{f}_d$ for all $d$. 
But by the width-based depth separation result \Cref{lem:Daniely}, we know this is impossible since $\tilde{f}_d$ has width subexponential in $d$. Therefore, contrary to our assumption, it must be the case that $R_2(\hat{f}_d)$ is exponential in $d$.
\end{proof}

\subsection{Approximating $f_d$ in the $L^2$-norm requires exponential $R_2$ cost}
\label{sec:lower bound on generalization}
In this section we adapt the construction of \cite{daniely2017depth} to the target function \begin{equation}
    f_d(\vx) = \psi_{3d}\left(\sqrt{d} \langle \vx^{(1)}, \vx^{(2)} \rangle\right)
\end{equation}
to prove that approximating $f_d$ in the $L^2$-norm to even constant error requires $R_2$ cost that is exponential in dimension:
\begin{lemma}
    \label{lem:need large R2 to approx in L2 norm}
    There exist real numbers $d_0,C > 0$ such that $d > d_0$ and $\Repregbias{2}(f) < 2^{Cd}$ implies that $\|f - f_d\|_{L^2}^2 \ge \frac{1}{50e^2 \pi^2}$.
\end{lemma}
After outlining the proof of this result, the remainder of this section establishes several auxiliary lemmas used in the proof.
\begin{proof}
    Similar to \cite{daniely2017depth}, we let $\mu_d$ denote the probability distribution obtained by pushing forward the uniform measure on $\bbS^{d-1}$ via the mapping $\vx \mapsto x_1$, and we use $N_{d,n}$ for the dimension of the set of spherical harmonics of order $n$ in $d$ dimensions.
    \Cref{lem:lower bound on generalization by R2 cost} adapts Theorem 4 in \cite{daniely2017depth} to show that
    for any $n \in \N$ and any $f \in \setofnns{2}$, 
    \begin{equation}
    \label{eq:lower bound on R2 in terms of generalization}
        4\sqrt{3}\Repregbias{2}(f) + 2\|f\|_{L^2}
        \ge \sqrt{N_{d,n}}\left(
            A_{d,n}(\psi_{3d}(\sqrt{d} \cdot)) - 
            \frac{\|f - f_d\|_{L^2}^2}{A_{d,n}(\psi_{3d}(\sqrt{d} \cdot))}\right)
    \end{equation}
    where $A_{d,n}(\psi_{3d}(\sqrt{d} \cdot))$ is the distance in the $L^2(\mu_d)$-norm of the function $t \mapsto \psi_{3d}(\sqrt{d} t)$ to the closest polynomial of degree less than $n$. 
    
    We choose $n=2d$. 
    In \Cref{lem:sawtooth far from polynomial}, we show that if $d$ is sufficiently large, then
    $
        A_{d,2d}(\psi_{3d}(\sqrt{d} \cdot) \ge \frac{1}{5e\pi};
    $
    that is, the sawtooth function is bounded away from being a polynomial of degree $2d-1$.
    If    
    $\|f - f_d\|_{L^2}^2 < \frac{1}{50e^2 \pi^2}$, then by the reverse triangle inequality
    \begin{equation}
        \|f\|_{L^2} 
        <  \|f_d\|_{L^2} + \|f - f_d\|_{L^2} 
        \le 1 + \frac{1}{5\sqrt{2}e \pi}.
    \end{equation}
    Plugging this all into \Cref{eq:lower bound on R2 in terms of generalization}, we get
    \begin{equation}
        4\sqrt{3}\Repregbias{2}(f) + 2 + \frac{2}{5\sqrt{2}e \pi}
        \ge \frac{\sqrt{N_{d,2d}}}{10e\pi}    
    \end{equation}
    whenever $\|f - f_d\|_{L^2}^2 < \frac{1}{50e^2 \pi^2}$.
    As shown in \Cref{lem:spherical harmonics of order 2d in d dims are 2^d}, $N_{d,2d} > 2^{d}$ for sufficiently large $d$. We conclude that there exist real numbers $d_0,C > 0$ such that $d > d_0$ and $\Repregbias{2}(f) < 2^{Cd}$ implies that $\|f - f_d\|_{L^2}^2 \ge \frac{1}{50e^2 \pi^2}$.
\end{proof}
\begin{lemma}
    \label{lem:lower bound on generalization by R2 cost}
    Consider a distribution $\xydist_d$ on $\gX_d \times [-1,1]$
    defined as
    \begin{align}
        \vx &\sim \Unif(\gX_d) \\
        y|\vx &= f_d(\vx)
    \end{align}
    for some function inner-product $f_d: \gX_d \rightarrow [-1,1]$ defined as $f_d(\vx) = g_d\left(\langle \vx^{(1)}, \vx^{(2)} \rangle\right)$.
    Then for all $f \in \setofnns{2}$ and $n \in \N$, 
    \begin{equation}
        \|f - f_d\|_{L^2}^2
        \ge A_{d,n}(g_d) \left( 
            A_{d,n}(g_d) - 
            \frac{4\sqrt{3}\Repregbias{2}(f) + 2\|f\|_{L^2}}
                 {\sqrt{N_{d,n}}}
        \right).
    \end{equation}
    where $A_{d,n}(g_d)$ is the distance in the $L^2(\mu_d)$-norm of the function $t \mapsto g_d(t)$ to the closest polynomial of degree less than $n$. 
\end{lemma}
\begin{proof}
    Let $\bm\phi = (\mW,\vb,\va,c)$ be an arbitrary parameterization of $f$ with $\|\vw_k\|_2^2 + |b_k|^2 = 1$ for each unit $k$. That is, $f(\vx) 
        = \sum_{k} a_k \left[\vw_k^\top \vx + b_k \right]_+ + c.$
    We now upper bound the $L^2$-norm of each ReLU unit in $\bm\phi$. By Cauchy-Schwarz, for all $\vx \in \gX_d$ we have 
    \begin{equation}
        |\vw_k^\top \vx + b_k| \le \sqrt{\|\vw_k\|_2^2 + |b_k|^2} \sqrt{\|\vx\|_2^2 + 1} = \sqrt{3}.
    \end{equation}
    Thus
    \begin{equation}
        \left\|a_k \left[\vw_k^\top \cdot + b_k \right]_+\right\|_{L^2}
        = \sqrt{\E_{\vx \sim \Unif(\gX_d)}
        \left[a_k^2 \left[\vw_k^\top\vx + b_k \right]_+^2\right]}
        \le \sqrt{3}|a_k|.
    \end{equation}
    Additionally, 
    \begin{align}
        \|c\|_{L^2}
        &= \left\|f - \sum_{k} a_k \left[\vw_k^\top \cdot + b_k \right]_+\right\|_{L^2} \\
        &\le  \left\|f\right\|_{L^2} + \sum_k \left\|a_k \left[\vw_k^\top \cdot + b_k \right]_+\right\|_{L^2} \\
        &\le  \left\|f\right\|_{L^2} + \sqrt{3} \sum_k |a_k|.
    \end{align}
    By Theorem 4 in \cite{daniely2017depth}, \begin{align}
        \|f - f_d\|_{L^2}^2
        &\ge A_{d,n}(g_d) \left( 
            A_{d,n}(g_d) - 
            \frac{
                2 \sum_k 
                    \left\|a_k \left[\vw_k^\top \cdot + b_k \right]_+\right\|_{L^2}  
                + 2 \|c\|_{L^2} 
            }
            {\sqrt{N_{d,n}}}
        \right) \\
        &\ge A_{d,n}(g_d) \left( 
            A_{d,n}(g_d) - 
            \frac{
                2 \left\|f\right\|_{L^2} + 4\sqrt{3} \sum_k |a_k|
            }
            {\sqrt{N_{d,n}}}
        \right) \label{eq:applying daniely theorem 4}
    \end{align}
    We now take the supremum of the right-hand side of \eqref{eq:applying daniely theorem 4} over all such parameterizations $\bm\phi$. By \Cref{lem:reduced depth-2 rep cost}, this gives the desired result.
\end{proof}
The next lemma is analogous to \cite[Lemma 5]{daniely2017depth} but for the sawtooth function instead of a sinusoid.
\begin{lemma}
\label{lem:sawtooth far from polynomial}
    If $d$ is sufficiently large, then $A_{d,2d}(\psi_{3d}(\sqrt{d} \cdot) \ge \frac{1}{5e\pi}$.
\end{lemma}
\begin{proof}
    By definition,
    \begin{equation}
        A_{d,2d}(\psi_{3d}(\sqrt{d} \cdot) := \min_{p \in \R[x;2d-1]} \|\psi_{3d}(\sqrt{d}\;\cdot) - p\|_{L^2(\mu_d)}
    \end{equation}
    where $\R[x;2d-1]$ denotes the set of polynomials of degree less than $2d$
    and $d\mu_d(t) := \frac{\Gamma(\frac{d}{2})}{\sqrt{\pi}\Gamma(\frac{d-1}{2})}(1-t^2)^\frac{d-3}{2}$.
    As shown in the proof of \cite[Lemma 5]{daniely2017depth}, for $|t| \le \frac{1}{\sqrt{d}}$ and $d$ sufficiently large, we have $d\mu_d(t) \ge \frac{\sqrt{d}}{2e\pi}$,
    and for all $p \in \R[x;2d-1]$ and $n \ge 2d-1$, 
    \begin{align}
        \|\psi_{n}(\sqrt{d}\;\cdot) - p\|_{L^2(\mu_d)}^2
        &= \int_{-1}^1 (\psi_{n}(\sqrt{d} t) - p(t))^2 d\mu_d(t)\\
        &\ge \frac{\sqrt{d}}{2e\pi}\int_{-d^{-1/2}}^{d^{-1/2}} (\psi_{n}(\sqrt{d} t) - p(t))^2 dt\\
        &= \frac{1}{2e\pi}\int_{-1}^1 (\psi_{n}(t) - p(t/\sqrt{d}))^2 dt.
    \end{align}
    Consider the intervals $I_i = (-1 + \frac{2i-2}{n}, -1 + \frac{2i}{n})$, $i = 1, \ldots n$, of width $2/n$. Each interval contains a full cycle of the sawtooth function. 
    Observe that $p(t/\sqrt{d})$ is a polynomial of degree at most $2d-1$, and so it has at most $2d-1$ roots in $[-1,1]$. On at least $n - 2d +1$ of the intervals $I_i$, the polynomial $p(t/\sqrt{d})$ does not change signs. On each interval $I_i$ where $p(t/\sqrt{d})$ does not change signs, $\psi_{n}$ is positive on 
    half of $I_i$ 
    and negative on 
    the other half of $I_i$.
    Thus, on 
    at least one subinterval of $I_i$ of width $1/n$, 
    $\psi_{n}(t)$ has the same sign as $p(t/\sqrt{d})$. It follows that
    \begin{align}
        \int_{-1}^1 (\psi_{n}(t) - p(t/\sqrt{d}))^2 dt
        &\ge (n - 2d +1) \int_0^{1/n} \psi_{n}^2(t) dt \\
        &= 2 (n - 2d +1) \int_0^{1/2n} (-2n t)^2 dt \\
        &= 2 (n - 2d +1) (2n)^2 \frac{1}{3(2n)^3} \\
        &= \frac{n - 2d +1}{3n}
    \end{align}
    where the first equality comes from the symmetry in $\psi_n$. 
    Thus $\|\psi_{n}(\sqrt{d}\;\cdot) - p\|^2_{L^2(\mu_d)} \ge \frac{n - 2d +1}{6ne\pi}$. In particular, choosing $n = 3d$ gives 
    \begin{equation}
        A_{d,2d}(\psi_{3d}(\sqrt{d} \cdot)^2 \ge \frac{d+1}{18de\pi} \ge \frac{1}{18e\pi} \ge \frac{1}{25e^2\pi^2}.
    \end{equation}
\end{proof}
\begin{lemma}
\label{lem:spherical harmonics of order 2d in d dims are 2^d}
$N_{d,2d} > 2^{d}$ for sufficiently large $d$.
\end{lemma}
\begin{proof}
The quantity $N_{d,n}$ is defined to be the dimension of the set of spherical harmonics of order $n$ in $d$ dimensions:
    \begin{equation}
        N_{d,n} := \frac{(2n + d - 2)(n+d-3)!}{n! (d-2)!}.
    \end{equation}
Using Stirling's approximation,
    \begin{align*}
        \lim_{d \rightarrow \infty} \frac{\log_2(N_{d,2d})}{d}
        &= \lim_{d \rightarrow \infty} \frac{\log_2(4d + d - 2) + \log_2(2d + d - 3)! - \log_2(2d)! - \log_2(d-2)!}{d} \\
        &= \lim_{d \rightarrow \infty} \frac{(3d - 3)\log_2(3d - 3) - (2d)\log_2(2d) - (d-2)\log_2(d-2)}{d} \\
        &> \lim_{d \rightarrow \infty} \frac{(3d-3)\log_2(2d) - (2d)\log_2(2d) - (d-2)\log_2(d)}{d} \\    
        &= \lim_{d \rightarrow \infty} \frac{d\log_2(2d) - d\log_2(d)}{d}  = 1.
    \end{align*}
Therefore there exists a $d_0$ such that $d \ge d_0$ implies $\frac{\log_2(N_{d,2d})}{d} > 1$.
\end{proof}

\subsection{Approximating $f_d$ in the $L^\infty$-norm with polynomial $R_3$ Cost}
\label{sec:existence of good approximator}

In this section, we show that there is a depth-3 network $f_{d,K}$ that well approximates $f_d(\vx) = \psi_{3d}\left(\sqrt{d} \langle \vx^{(1)}, \vx^{(2)} \rangle\right)$ and bound its $R_3$ cost. 
The sawtooth function $\psi_n$ can be expressed as a depth-2 network of width $2n + 2$ as follows:
\begin{equation}
\label{eq:sawtooth as a relu}
    \psi_n(t) = -2n[t + 1]_+ + 2n[ t - 1]_+
    + 4n \sum_{j=1}^n (-1)^{j+n+1} \left[ t - \frac{2j-1}{2n}\right]_+ + (-1)^{j+n} \left[ t + \frac{2j-1}{2n}\right]_+.
\end{equation}
\begin{figure}[ht!]
    \centering
\includegraphics[width=0.6\textwidth]{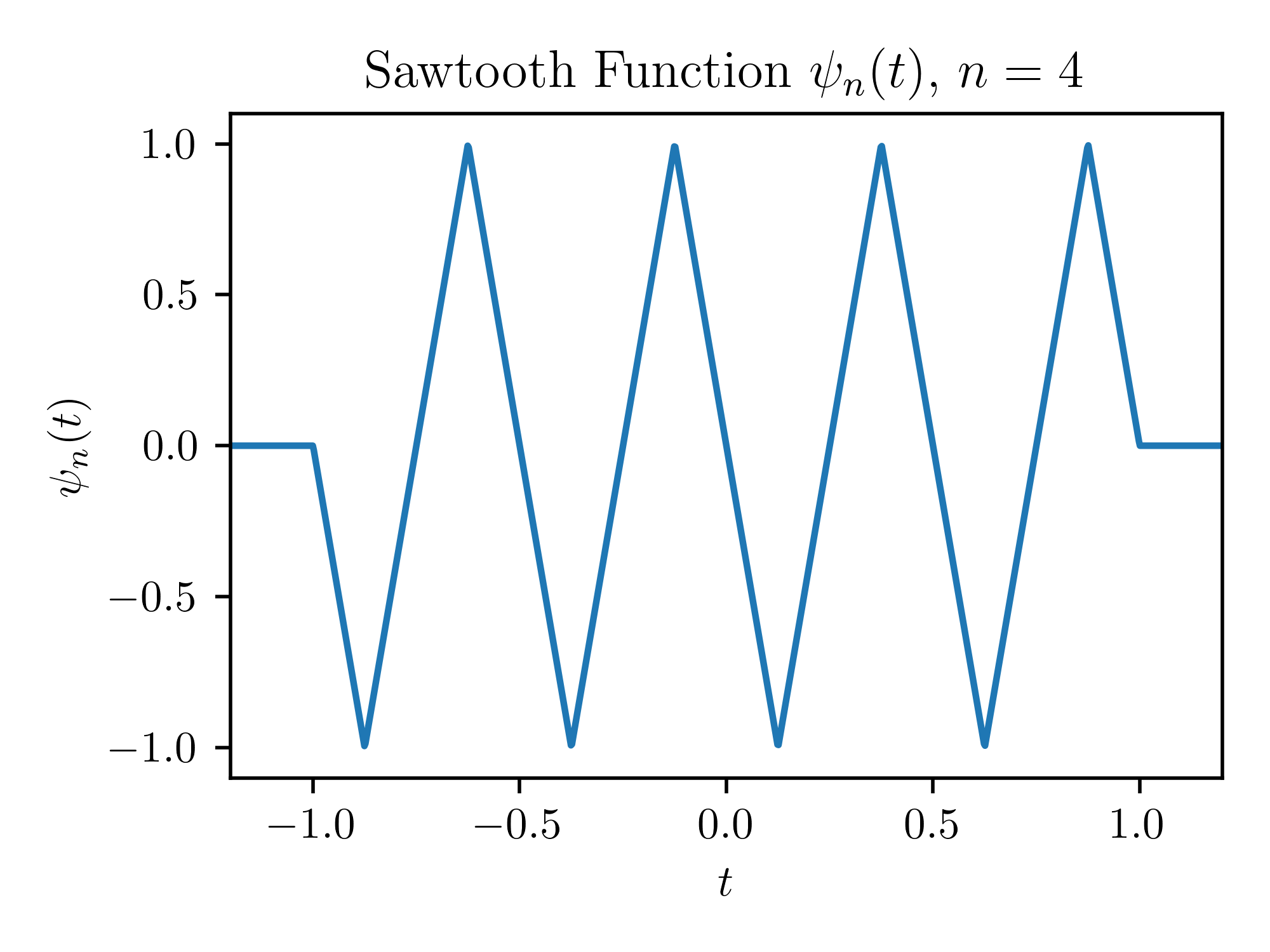}
      \caption{The sawtooth function $\psi_{n}: \R \rightarrow [-1,1]$ with $n=4$. The function $\psi_{n}$ has $n$ cycles in $[-1,1]$ and is equal to zero outside $[-1,1]$.}
      \label{fig:sawtooth}
    \end{figure}
\begin{lemma}
\label{lem:param of sawtooth}
    For all scalars $\beta > 0$ and $n \in \N$, there are vectors $\va,\vu,\vq \in \R^{2n+2}$ such that $\psi_n(\beta t) = \va^\top[\vu t + \vq]_+$, $\vu^\top\vq = 0$, $\|\vu\| = 1$, and $\|\va\|^2 + \|\vq\|^2 = O(n^4 \beta^2 + \beta^{-2})$.
\end{lemma}
\begin{proof}
    Denote the vector of all ones by $\vone.$
    Using \Cref{eq:sawtooth as a relu}, define vectors $\va_0,\vu_0,\vq_0 \in \R^{2n+2}$ so that
    $\psi_n(\beta t) = \va_0^\top[\vu_0 t + \vq_0]_+$
    where $\vu_0 = \beta \vone$,
    \begin{equation}
        \vq_0 = \begin{bmatrix}
            1 , &
            -1 , &
            -\frac{1}{2n} , &
            \frac{1}{2n} , &
            -\frac{3}{2n} , &
            \frac{3}{2n} , &
            \cdots, &
            -\frac{2n-1}{2n} , &
            \frac{2n-1}{2n}
        \end{bmatrix}^\top, 
    \end{equation}
    and
    \begin{equation}
        \va_0 = \begin{bmatrix}
            -2n , &
            2n , &
            \pm4n , &
            \pm4n  , &
            \cdots  , &
            \pm4n  , &
            \pm4n
        \end{bmatrix}^\top.
    \end{equation}
    Observe that 
    \begin{align}
        \|\vu_0\|^2 &= (2n+2) \beta^2 \\
        \|\vq_0\|^2 &\le (2n+2) \\ 
        \|\va_0\|^2 &\le (2n+2) 16n^2. 
    \end{align}
    Let $\vu = \frac{\vu_0}{ \|\vu_0\|}$, $\vq = \frac{\vq_0}{ \|\vu_0\|}$, and 
    $\va =  \|\vu_0\| \va_0$. Then 
    $\psi_n(\beta t) = \va^\top[\vu t + \vq]_+$ by the homogeneity of ReLU. We also observe that $\|\vu\| = 1$ and $\vu^\top\vq = 0$. Finally, 
    \begin{equation}
        \|\va\|^2 + \|\vq\|^2
        \le (2n+2)^2 16n^2 \beta^2 + \beta^{-2}
        = O(n^4 \beta^2 + \beta^{-2}).
    \end{equation}
\end{proof}
The next two lemmas allow us to get an approximation of the inner product by approximating the square function.
\begin{lemma}
\label{lem:approximating the square}
    For all $s > 0$ and $K \in \N$, 
    the function $f_{square}^{s,K}(t) := \frac{2s}{K} \sum_{k=1}^{K} [t - \frac{sk}{K}]_+ + [-t - \frac{sk}{K}]_+$ with $2K$ ReLU units
    satisfies
    \[
        \sup_{t \in [-s,s]} |f_{square}^{s,K}(t) - t^2| \le s^2\left(\frac{1}{K} + \frac{1}{K^2}\right).
    \]
\end{lemma}
\begin{proof}
    Observe that $f_{square}^{s,K}(-t) = f_{square}^{s,K}(t)$, 
    so it suffices to consider $t \in [0,s]$. Given $t \in [0,s]$, all of the $[-t - \frac{sk}{K}]_+$ terms in $f_{square}^{s,K}$ are equal to zero, and the $[t - \frac{sk}{K}]_+$ terms are nonzero if and only if $k < \frac{Kt}{s}$. That is, 
    \[
        f_{square}^{s,K}(t) 
        = \frac{2s}{K} \sum_{k=1}^{\floor{\frac{Kt}{s}}} \left(t - \frac{sk}{K}\right).
    \]
    We use the summation formula $\sum_{j=1}^{n} j = \frac{n(n+1)}{2}$ and the notation $\{x\} := x - \floor{x} \in [0,1)$ to show that this quantity is approximately $t^2$; it is straightforward to verify that
    \begin{align*}
        f_{square}^{s,K}(t) 
        &= 
        t^2
        - \frac{st}{K} 
            - \frac{s^2}{K^2} \left\{\frac{Kt}{s}\right\} \left(
                \left\{\frac{Kt}{s}\right\} - 1 
            \right).
    \end{align*}
    Thus,
    \begin{equation}
        \sup_{t \in [-s,s]} |f_{square}^{s,K}(t) - t^2| 
        = \sup_{t \in [0,s]} \left|\frac{st}{K} 
            + \frac{s^2}{K^2} \left\{\frac{Kt}{s}\right\} \left(
                \left\{\frac{Kt}{s}\right\} - 1 
            \right) \right|
        \le \frac{s^2}{K} + \frac{s^2}{K^2}.
    \end{equation}
\end{proof}
\begin{lemma}
\label{lem:approximating the inner product}
    The function 
    \begin{equation}
        f_{inner}^K(\vx)
            := \sum_{i=1}^d f_{square}^{\sqrt{2},K}\left(\frac{1}{\sqrt{2}}
                \begin{bmatrix}
                    \ve_i \\ \ve_i
                \end{bmatrix}^\top 
            \vx\right) - 1
    \end{equation}
    satisfies
    \begin{equation}
        \sup_{\vx \in \gX_d} |f_{inner}^K(\vx) - \langle \vx^{(1)}, \vx^{(2)} \rangle| \le 2d \left(\frac{1}{K} + \frac{1}{K^2}\right).
    \end{equation}
    Further, for any scalar $\beta > 0$, the function $\beta^{-1} f_{inner}^K(\vx)$ is in $\setofnns{2,2Kd}$ and 
    \begin{equation}
        R_2(\beta^{-1} f_{inner}^K(\vx);2Kd) = O(d \beta^{-1} + \beta^{-2}).
    \end{equation}
\end{lemma}
\begin{proof}
    Fix $\vx \in \gX_d$.
    Similarly to Corollary 7 in \cite{daniely2017depth}, observe that
    \[
        \langle \vx^{(1)}, \vx^{(2)} \rangle
        =
        \sum_{i=1}^d \left(\frac{1}{\sqrt{2}}
            \begin{bmatrix}
                \ve_i \\ \ve_i
            \end{bmatrix}^\top 
        \vx\right)^2 - 1.
    \]
    Additionally,
    \[
        \left|\frac{1}{\sqrt{2}}
            \begin{bmatrix}
                \ve_i \\ \ve_i
            \end{bmatrix}^\top 
        \vx\right| \le \|\vx\|_2 = \sqrt{2}.
    \]
    Then
     \begin{align*}
        \sup_{\vx \in \gX_d}|f_{inner}^K(\vx) - \langle \vx^{(1)}, \vx^{(2)} \rangle| 
        &\le \sup_{\vx \in \gX_d}\sum_{i=1}^d \left|
            f_{square}^{\sqrt{2},K}\left(\frac{1}{\sqrt{2}}
                \begin{bmatrix}
                    \ve_i \\ \ve_i
                \end{bmatrix}^\top 
            \vx\right)
            - \left(\frac{1}{\sqrt{2}}
                \begin{bmatrix}
                    \ve_i \\ \ve_i
                \end{bmatrix}^\top 
            \vx\right)^2
        \right| \\
        &\le d \sup_{|t| \le \sqrt{2}} \left|
            f_{square}^{\sqrt{2},K}(t)
            - t^2
        \right| \\
        &\le 2d \left(\frac{1}{K} + \frac{1}{K^2}\right).
    \end{align*}
    Now fix $\beta > 0$. Since 
    \begin{align}
        \frac{1}{\beta} f_{inner}^K(\vx) 
        &= \frac{1}{\beta} \sum_{i=1}^d f_{square}^{\sqrt{2},K}
        \left(
            \frac{1}{\sqrt{2}}
                \begin{bmatrix}
                    \ve_i \\ \ve_i
                \end{bmatrix}^\top 
            \vx
        \right) 
        - \frac{1}{\beta} \\
        &= \frac{2\sqrt{2}}{\beta K} \sum_{i=1}^d \sum_{k=1}^{K} 
        \left[\frac{1}{\sqrt{2}}
                \begin{bmatrix}
                    \ve_i \\ \ve_i
                \end{bmatrix}^\top 
            \vx
            - \frac{\sqrt{2}k}{K}
        \right]_+ 
        + 
        \left[-\frac{1}{\sqrt{2}}
                \begin{bmatrix}
                    \ve_i \\ \ve_i
                \end{bmatrix}^\top 
            \vx
            - \frac{\sqrt{2}k}{K}
        \right]_+
        - \frac{1}{\beta}
    \end{align}
    we see that $\beta^{-1} f_{inner}^K(\vx) \in \setofnns{2,2Kd}$. Finally, we apply \Cref{lem:reduced depth-2 rep cost} to get
    \begin{align}
        R_2(\beta^{-1} f_{inner}^K(\vx);2Kd)
        &\le \sum_{i=1}^d \sum_{k=1}^{K} \frac{4\sqrt{2}}{\beta K} \sqrt{1 + \frac{2k^2}{K^2}} + \frac{1}{2\beta^2} \\
        &\le \frac{4\sqrt{3}d}{\beta} + \frac{1}{2\beta^2}.
    \end{align}
\end{proof}
Finally, we use $f_{inner}^K$ to construct $f_{d,K}$ and bound its $R_3$ cost.
\begin{lemma}
\label{lem:approximation of target with small error and cost}
    Let $f_d(\vx) = \psi_{3d}\left(\sqrt{d} \langle \vx^{(1)}, \vx^{(2)} \rangle\right)$.
    For all $K \in \N$, 
    there is a depth-3 neural network $f_{d,K}$ of width $\width_{d,K} := \max(6d+2,2Kd)$ such that 
    $
        \|f_d - f_{d,K}\|_{L^\infty}
        = O\left( \frac{d^{5/2}}{K} \right)
    $
    and 
    $
        \Repregbias{3}(f_{d,K};\width_{d,K}) = O(d^{5/2}).
    $
\end{lemma}
\begin{proof}
    Choose $f_{d,K}(\vx) := \psi_{3d}(\sqrt{d}f_{inner}^K(\vx))$, which can be expressed as a depth-3 network with hidden widths $2Kd$ and $6d+2$. 
    For all $\vx \in \gX_d$, we use the fact that $\psi_n$ is $2n$-Lipschitz to see that
    \begin{align*}
        \|f_d - f_{d,K}\|_{L^\infty}
        &= \sup_{\vx \in \gX_d}|\psi_{3d}(\sqrt{d}f_{inner}^K(\vx)) - \psi_{3d}(\sqrt{d}\langle \vx^{(1)}, \vx^{(2)} \rangle)|  \\
        &\le 6d\sqrt{d} \sup_{\vx \in \gX_d}|f_{inner}^K(\vx) - \langle \vx^{(1)}, \vx^{(2)} \rangle| \\
        &\le 12d^{5/2} \left(\frac{1}{K} + \frac{1}{K^2}\right).
    \end{align*} 
    We now bound $\Repregbias{3}(f_{d,K};\width_{d,K}).$
    Notice that $f_{d,K}$ can be expressed as $f_{d,K} = h \circ g$ where we set $h:\R\rightarrow\R$ to be $h(t) = \psi_{3d}(\sqrt{d}\beta t)$ and $g:\gX_d \rightarrow \R$ to be $g(\vx) = \beta^{-1}f_{inner}^K(\vx)$ where $\beta>0$ is a value we will optimize over later. 
    By \Cref{lem:param of sawtooth} there are vectors $\va,\vu,\vq \in \R^{2n+2}$ such that $h(t) = \va^\top[\vu t + \vq]_+$, $\vu^\top\vq = 0$, $\|\vu\|_2 = 1$, and $\|\va\|_2^2 + \|\vq\|_2^2 = O(d^5 \beta^2 + \beta^{-2}d^{-1})$.
    
    Let $\bm\phi_g = (\mW,\vb,\vv,c)$ be an arbitrary parameterization of $g$ of width $2Kd$, so that
    $
        g(\vx) = \vv^\top[\mW \vx +\vb]_+ +  c.
    $
    This gives a parameterization $\bm\phi_f$ of $f_{d,K}$ as
    \[
    f_{d,K}(\vx) = \va^\top[\vu\vv^\top[\mW \vx +\vb]_+ + (c\vu +\vq)]_+.
    \]
    Using the properties of $\va,\vu$ and $\vq$, we see that
    \begin{align}
    \|\bm\phi_f\|^2 
    &= \|\va\|_2^2 + \|\vu\|_2^2\|\vv\|_2^2 + \|\mW\|_F^2 + \|\vb\|_2^2 + c^2\|\vu\|_2^2 + \|\vq\|_2^2 \\
    &=  O(d^5 \beta^2 + \beta^{-2}d^{-1}) + \|\bm\phi_g\|^2.
    \end{align}
    Minimizing over parameterizations and using \Cref{lem:approximating the inner product}, we get
    \begin{align}
        \Repregbias{3}(f_{d,K};\width_{d,K}) 
        &\leq O(d^5 \beta^2 + \beta^{-2}d^{-1})
            + \frac{2}{3}\Repregbias{2}(g;2Kd) \\
        &= O(d^5 \beta^2 + \beta^{-2}d^{-1} 
            + d \beta^{-1} + \beta^{-2})
    \end{align}
    Choosing $\beta = d^{-5/4}$ gives 
    $
    \Repregbias{3}(f_{d,K};\width_{d,K}) = O(d^{5/2}).
    $
    None of the constants hidden in the big-O depend on $K$.
\end{proof}

\subsection{Existence of interpolants with mild $R_2$ cost}
\label{sec:interpolant exists}
In this section, we will prove that with high probability over the samples, an interpolant exists with $R_2$ cost that depends only mildly on the number of samples (\Cref{lem:tent interpolator is cheap}). 
To do this, we show that with high probability the samples are sufficiently separated (\Cref{lem:whp samples are separated}), and then show that separated samples on $\gX_d$ can each be assigned a hyperplane that is sufficiently far away from any other sample  (\Cref{lem:tent interpolator exists}). We start with the following simple bound on the Beta function. 
\begin{lemma}
\label{lem:bound on beta term}
For all $d \ge 3$,
\begin{equation}
    B\left(\frac{d-1}{2},\frac{1}{2}\right)\ge \frac{2\sqrt{\pi}}{d-1}.
\end{equation}
\end{lemma}
\begin{proof}
    Using the identity $z\Gamma(z) = \Gamma(z+1)$ and the fact that $\Gamma(z)$ is an increasing function on the domain $z \ge \frac{3}{2}$, we see that
    \begin{equation}
        (d-1) B\left(\frac{d-1}{2},\frac{1}{2}\right) 
        = 2\frac{\frac{d-1}{2}\Gamma(\frac{d-1}{2})\Gamma(\frac{1}{2})}{\Gamma(\frac{d}{2})}
        = 2\frac{\Gamma(\frac{d+1}{2})\sqrt{\pi}}{\Gamma(\frac{d}{2})}
        \ge 2\sqrt{\pi}.
    \end{equation}
\end{proof}
\begin{lemma}
\label{lem:whp samples are separated}
Let $\vx_1, \ldots, \vx_m$ be i.i.d. samples from $\Unif(\gX_d)$. Then for $\eta < 1$,
\[\P(\min_{i\ne j} \|\vx_i - \vx_j\|_2 \le \eta) 
< m^2 \eta^{d-1}\]
\end{lemma}
\begin{proof}
    We first consider the distance between $\vx_1$ and $\vx_2$.
    Since $\|\vx_1^{(1)} - \vx_2^{(1)}\|_2 \le \|\vx_1 - \vx_2\|_2$
    it follows that
    \begin{equation}
        \P(\|\vx_1 - \vx_2\|_2 \le \eta) \le \P(\|\vx_1^{(1)} - \vx_2^{(1)}\|_2 \le \eta).
    \end{equation}
    As shown in \cite{Sidiropoulos_2014}), the probability density function of $\|\vx_1^{(1)} - \vx_2^{(1)}\|_2$ is
    \begin{equation}
        \P(\|\vx_1^{(1)} - \vx_2^{(1)}\|_2 = \eta) = \frac{\eta\left(\eta^2 - \frac{\eta^4}{4}\right)^{\frac{d-3}{2}}}{B\left(\frac{d-1}{2},\frac{1}{2}\right)}.
    \end{equation}
    Integrating and using the bound on the Beta function from \Cref{lem:bound on beta term}, we get 
    \begin{align}
        \P(\|\vx_1^{(1)} - \vx_2^{(1)}\|_2 \le \eta)
        &= \frac{1}{B\left(\frac{d-1}{2},\frac{1}{2}\right)}\int_0^\eta t\left(t^2 - \frac{t^4}{4}\right)^{\frac{d-3}{2}}dt \\
        &\le \frac{d-1}{2\sqrt{\pi}}\int_0^\eta t\left(t^2\right)^{\frac{d-3}{2}}dt \\
        &= \frac{d-1}{2\sqrt{\pi}}\int_0^\eta t^{d-2}dt \\
        &< \eta^{d-1}.
    \end{align}
    Finally, there are $\binom{m}{2}$ pairwise distances between the samples, so we can use the union bound to get 
    \begin{equation}
        \P(\min_{i\ne j} \|\vx_i - \vx_j\|_2 \le \eta) 
    < \binom{m}{2}\P(\|\vx_1 - \vx_2\|_2 \le \eta) <  m^2\eta^{d-1}.
    \end{equation}
\end{proof}
\begin{lemma}
\label{lem:tent interpolator exists}
    For any finite set of points $\{\vx_j\}_{j=1}^m \subseteq \gX_d$ that are $\eta$-separated, there exists a unit vector $\vv_j \in \R^{2d}$ for all $j\in [m]$ such that $\vx_j$ is contained in the hyperplane $\{\vx \in \R^{2d} : \vv_j^\top \vx = \sqrt{2}\}$ and $\vx_j$ is the only point contained in the set $T_j := \{\vx \in \R^{2d} : |\vv_j^\top\vx - \sqrt{2}| <\frac{\eta^2}{2\sqrt{2}} \}$. 
\end{lemma}
\begin{proof}
    Assume $\{\vx_j\}_{j=1}^m \subseteq \gX_d$ and $\min_{i\ne j} \|\vx_i - \vx_j\|_2 \ge \eta$. 
     Choose $\vv_j = \frac{1}{\sqrt{2}}\vx_j$. Clearly $\|\vv_j\|_2 = 1$, and
     \begin{equation}
        \vv_j^\top \vx_j = \frac{1}{\sqrt{2}} \|\vx_j\|^2 = \sqrt{2}.
     \end{equation}
     If $i \ne j$, then observe that
    \begin{equation}
        \eta^2 \le \|\vx_i - \vx_j\|_2^2 = \|\vx_i\|^2 +  \|\vx_j\|^2 - 2 \vx_i^\top \vx_j = 4 - 2 \vx_i^\top \vx_j.
    \end{equation}
    Hence,
     \begin{equation}
        |\vv_j^\top \vx_i - \sqrt{2} | = \left|\frac{1}{\sqrt{2}} \vx_j^\top \vx_i - \sqrt{2}\right|
        \ge \frac{\eta^2}{2\sqrt{2}}.
     \end{equation}
\end{proof}
We now have the pieces we need for the proof of \Cref{lem:tent interpolator is cheap}. 
\begin{lemma}
\label{lem:tent interpolator is cheap}
    Consider a distribution $\xydist_d$ on $\gX_d \times [-1,1]$
    defined as
    \begin{align}
        \vx &\sim \Unif(\gX_d) \\
        y|\vx &= f_d(\vx)
    \end{align}
    for some function $f_d: \gX_d \rightarrow [-1,1]$.
    Given a sample $S = \{(\vx_i,y_i)\}_{i=1}^m$ of size $m$ drawn i.i.d. from $\xydist_d$, 
    with probability at least $1-\delta$
    there exists an interpolant $\hat f$ of $S$ such that 
    $\Repregbias{2}(\hat f) \le 16 \sqrt{2}|S|^{\frac{d+3}{d-1}} \delta^{-\frac{2}{d-1}}$.
\end{lemma}
\begin{proof}
    By \Cref{lem:whp samples are separated}, the data is
    $\delta^{\frac{1}{d-1}} |S|^{\frac{-2}{d-1}}$ separated with probability at least $1-\delta$. For convenience, let $\eta = \delta^{\frac{1}{d-1}} |S|^{\frac{-2}{d-1}}$ and $\eta_0 = \frac{\eta^2}{2\sqrt{2}}$. Note that $\eta, \eta_0 \in (0,1)$.
    
    Consider the function $z_{\eta_0}:\R\rightarrow \R$ defined by $z_{\eta_0}(t) = \eta_0^{-1}([t-\eta_0]_+ -2[t]_+ + [t+\eta_0]_+)$, which vanishes for $|t| > \eta_0$, and is such that $z_{\eta_0}(0) = 1$. 
    By \Cref{lem:tent interpolator exists}, for all $j \in [n]$ 
    there exists a unit vector $\vv_j \in \R^{2d}$ for all $j\in [n]$ such that $\vx_j$ is contained in the hyperplane $\{\vx \in \R^{2d} : \vv_j^\top \vx = \sqrt{2}\}$ and $\vx_j$ is the only training point contained in the set $T_j := \{\vx \in \R^{2d} : |\vv_j^\top\vx - \sqrt{2}| <\eta_0 \}$.
    Define the ridge function $r_j:\R^{2d}\rightarrow \R$ by the depth-2 network of width 3 as follows:
    \begin{equation}
        r_j(\vx) 
        = z_{\eta_0}(\vv_j^\top\vx -\sqrt{2})
        = \eta_0^{-1}([\vv_j^\top\vx -\sqrt{2}-\eta_0]_+ -2[\vv_j^\top\vx -\sqrt{2}]_+ + [\vv_j^\top\vx -\sqrt{2}+\eta_0]_+).
    \end{equation}
    Since the support of $r_j$ coincides with $T_j$, and $\vv_j^\top\vx_j -\sqrt{2} = 0$, we see that $r_j(\vx_i) = \delta_{ij}$. Therefore, the width $3|S|$, depth-2 network $\hat f(\vx) = \sum_{j=1}^{|S|} y_j r_j(\vx)$ interpolates the samples.

    Using \Cref{lem:reduced depth-2 rep cost},
    \begin{align}
        \Repregbias{2}\left(\hat f ; 3|S|\right) 
        &\le \sum_{j=1}^{|S|} |y_j| \eta_0^{-1}\left(\sqrt{1+(\sqrt{2} + \eta_0)^2} + 2\sqrt{3} + \sqrt{1+(-\sqrt{2} + \eta_0)^2}\right)\\
        &\le 8|S|\eta_0^{-1} \\
        &= 16 \sqrt{2}|S|^{\frac{d+3}{d-1}} \delta^{-\frac{2}{d-1}}.
    \end{align}
\end{proof}

\subsection{Estimation error bound for depth-3 networks}
\label{sec:rademacher bounds}
In this section, we present an estimation error bound (\Cref{lem:estimation error bound}) derived from the Rademacher complexity bounds in \cite{neyshabur2015norm}. We begin with several auxiliary lemmas. Given a depth-3 network $f_{\bm\phi} \in \setofnns{3}$,  this first lemma rewrites $f_{\bm\phi}$ so that it will be compatible with the framework in \cite{neyshabur2015norm}.
\begin{lemma}
\label{lem:rewriting as a wider network}
    If $\bm\phi = (\mW_1,\vb_1,\mW_{2},\vb_{2},\vw_3,b_3)$ and 
    $\frac{1}{3}\|\bm\phi\|^2 \le M$,
    then
    \begin{equation}
    \label{eq:rewriting as a wider network}
        f_{\bm\phi}(\vx)
        = 
        \begin{bmatrix}
            \vw_3^\top & b_3
        \end{bmatrix} \left[\begin{bmatrix}
            \mW_2^\top & \vb_2 \\
            \vzero     & 1
        \end{bmatrix} \left[\begin{bmatrix}
            \mW_1^\top & \vb_1 \\
            \vzero     & 1
        \end{bmatrix} \begin{bmatrix}
            \vx \\ 1
        \end{bmatrix}
        \right]_+ \right]_+
    \end{equation}
    with 
    \[
        \left\|\begin{bmatrix}
                \vw_3^\top & b_3
            \end{bmatrix}\right\|_2 \left\| \begin{bmatrix}
                \mW_2^\top & \vb_2 \\
                \vzero     & 1
            \end{bmatrix} \right\|_F \left \|\begin{bmatrix}
                \mW_1^\top & \vb_1 \\
                \vzero     & 1
            \end{bmatrix} \right\|_F
        \le \left(M+\frac{2}{3}\right)^{3/2}.
    \]
\end{lemma}
\begin{proof}
    It is straightforward to verify \Cref{eq:rewriting as a wider network}.
    Observe that 
    \begin{align*}
        M 
        &\ge \frac{1}{3}\|\bm\phi\|^2
        = \frac{1}{3}\left(
            \left\|\begin{bmatrix}
                \vw_3^\top & b_3
            \end{bmatrix}\right\|_2^2 + \left\| \begin{bmatrix}
                \mW_2^\top & \vb_2 \\
                \vzero     & 1
            \end{bmatrix} \right\|_F^2 + \left \|\begin{bmatrix}
                \mW_1^\top & \vb_1 \\
                \vzero     & 1
            \end{bmatrix} \right\|_F^2 - 2
        \right) \\
        &\ge -\frac{2}{3} + \left(
            \left\|\begin{bmatrix}
                \vw_3^\top & b_3
            \end{bmatrix}\right\|_2 \left\| \begin{bmatrix}
                \mW_2^\top & \vb_2 \\
                \vzero     & 1
            \end{bmatrix} \right\|_F \left \|\begin{bmatrix}
                \mW_1^\top & \vb_1 \\
                \vzero     & 1
            \end{bmatrix} \right\|_F
        \right)^{2/3}.
    \end{align*}
    where the second inequality comes from the AM-GM inequality.
\end{proof}
We now apply Theorem 1 in \cite{neyshabur2015norm} to get a bound on the Rademacher complexity of the set of depth-3 networks with representation cost bounded by $M$ with respect to $\Unif(\gX_d)$. We use $\setofnns{3}^M$ to denote this set:
\begin{equation}
    \setofnns{3}^M := \{f \in \setofnns{3} : \Repregbias{3}(f) \le M\}.
\end{equation}
Given a function class $\gH$, we write $\Rad_m(\gH;(\vx_i)_{i=1}^m)$ for the empirical Rademacher complexity with respect to samples $(\vx_i)_{i=1}^m$. That is, 
\begin{equation}
    \Rad_m(\gH;(\vx_i)_{i=1}^m) 
        := \E_{\xi \sim \{\pm 1\}^m}\left[\sup_{h \in \gH} \frac{1}{m}\left|\sum_{i=1}^m \xi_i h(\vx_i)\right|\right]
\end{equation}
where $\xi \sim \{\pm 1\}^m$ denotes that each entry in $\xi$ is an iid draw from $\Unif\{\pm1\}$. 
We write $\Rad_{\gX_d^m}(\gH)$ for the Rademacher complexity of $\gH$ with respect to $m$ i.i.d. samples from $\Unif(\gX_d)$:
\begin{equation}
    \Rad_{\gX_d^m}(\gH) 
    := \E_{(\vx_i)_{i=1}^m \overset{iid}{\sim} \Unif(\gX_d)} [\Rad_m(\gH;(\vx_i)_{i=1}^m].
\end{equation}
\begin{lemma}[Rademacher Complexity Bound]
\label{lem:rademacher complexity bound}
    $\Rad_{\gX_d^m}(\setofnns{3}^M) = O\left(\frac{M^{3/2}}{m^{1/2}}\right)$.
\end{lemma}
\begin{proof}
    Theorem 1 in \cite{neyshabur2015norm} bounds the empirical Rademacher complexity of 
    \begin{equation}
        \gN^{3,\text{dim}=D}_{\gamma_{22} \le \gamma}
        := \{f: \R^D \rightarrow \R | f(\vx) = \vw_3^\top \left[\mW_2\left[\mW_1 \vx \right]_+\right]_+, \|\vw_3\|_2\|\mW_2\|_F\|\mW_1\|_F \le \gamma\}
    \end{equation}
    as 
    \begin{equation}
        \Rad_m(\gN^{3,\text{dim}=D}_{\gamma_{22} \le \gamma};(\vx_i)_{i=1}^m) 
        \le \frac{4\sqrt{2} \gamma  \max_i \|\vx_i\|_2}{\sqrt{m}}. 
    \end{equation}
    By \Cref{lem:rewriting as a wider network}, 
    $
        \setofnns{3}^M \subseteq \gN^{3,\text{dim}=D}_{\gamma_{22} \le \gamma}
    $ with $D=2d+1$ and $\gamma = \left(M+\frac{2}{3}\right)^{3/2}$. 
    Therefore,
    \[
        \Rad_m(\setofnns{3}^M;(\vx_i)_{i=1}^m) 
        \le \frac{4\sqrt{2} \left(M+\frac{2}{3}\right)^{3/2}  \max_i \sqrt{1+\|\vx_i\|_2^2}}{\sqrt{m}}.
    \]
    where we have replaced $\|\vx_i\|_2$ with $\sqrt{1+\|\vx_i\|_2^2}$ because $\setofnns{3}^M$ is embedded in $\gN^{3,\text{dim}=D}_{\gamma_{22} \le \gamma}$ by extending in the input $\vx \in \R^{2d}$ to $\begin{bmatrix}
            \vx^\top & 1
        \end{bmatrix}^\top \in \R^{2d+1}$.
    Since all samples $\vx_i \sim \Unif(\gX_d)$ have norm $\sqrt{2}$, we get 
    \[
        \Rad_{\gX_d^m}(\setofnns{3}^M)
        \le \frac{4\sqrt{2} \left(M+\frac{2}{3}\right)^{3/2}\sqrt{3}}{\sqrt{m}} = O\left(\frac{M^{3/2}}{m^{1/2}}\right).
    \]
\end{proof}
The other piece we need for an estimation error bound is to uniformly bound $\|f_d - h\|_{L^\infty}$ over $\setofnns{3}^M$. 
\begin{lemma}
\label{lem:uniform loss bound}
    If $f_d : \gX_d \rightarrow [-1,1]$, then
    $\sup_{h \in \setofnns{3}^M} \|f_d - h\|_{L^\infty} = O(M^{3/2})$.
\end{lemma}
\begin{proof}
    If $h \in \setofnns{3}^M$, then by \Cref{lem:rewriting as a wider network},
    \[
        h(\vx)
        = 
        \begin{bmatrix}
            \vw_3^\top & b_3
        \end{bmatrix} \left[\begin{bmatrix}
            \mW_2^\top & \vb_2 \\
            \vzero     & 1
        \end{bmatrix} \left[\begin{bmatrix}
            \mW_1^\top & \vb_1 \\
            \vzero     & 1
        \end{bmatrix} \begin{bmatrix}
            \vx \\ 1
        \end{bmatrix}
        \right]_+ \right]_+
    \]
    for some parameterization $\bm\phi = (\mW_1,\vb_1,\mW_{2},\vb_{2},\vw_3,b_3)$ with
    \[
        \left\|\begin{bmatrix}
                \vw_3^\top & b_3
            \end{bmatrix}\right\|_2 \left\| \begin{bmatrix}
                \mW_2^\top & \vb_2 \\
                \vzero     & 1
            \end{bmatrix} \right\|_F \left \|\begin{bmatrix}
                \mW_1^\top & \vb_1 \\
                \vzero     & 1
            \end{bmatrix} \right\|_F 
        \le \left(M+\frac{2}{3}\right)^{3/2}.
    \] 
    Because $\|\mA\mB\|_F \le \|\mA\|_F\|\mB\|_F$ and $\|[\mA]_+\|_F \le \|\mA\|_F$, we see that for $\vx \in \gX_d$,
    \[
        |h(\vx)|
        \le 
        \left\|\begin{bmatrix}
            \vw_3^\top & b_3
        \end{bmatrix}\right\|_2 \left\| \begin{bmatrix}
            \mW_2^\top & \vb_2 \\
            \vzero     & 1
        \end{bmatrix} \right\|_F \left \|\begin{bmatrix}
            \mW_1^\top & \vb_1 \\
            \vzero     & 1
        \end{bmatrix} \right\|_F \left\|\begin{bmatrix}
            \vx \\ 1
        \end{bmatrix}
        \right\|_2
        \le \sqrt{3}
        \left(M+\frac{2}{3}\right)^{3/2}.
    \]
    This shows that 
    \[
        \sup_{h \in \setofnns{3}^M} \|f_d - h\|_{L^\infty} 
        \le \|f_d\|_{L^\infty} + \sup_{h \in \setofnns{3}^M} \|h\|_{L^\infty} 
        \le 1 + \sqrt{3}\left(M+\frac{2}{3}\right)^{3/2}
        = O(M^{3/2}).
    \]
\end{proof}
Using \Cref{lem:rademacher complexity bound,lem:uniform loss bound}, standard Rademacher complexity arguments yield an estimation error bound over $\setofnns{3}^M$, as shown in the following lemma. 
\begin{lemma}
    \label{lem:estimation error bound}
    Consider a distribution $\xydist_d$ on $\gX_d \times [-1,1]$
    defined as
    \begin{align}
        \vx &\sim \Unif(\gX_d) \\
        y|\vx &= f_d(\vx)
    \end{align}
    for some function $f_d: \gX_d \rightarrow [-1,1]$ .
    If $f \in \setofnns{3}$ with $R_3(f) \le M$, then
    \begin{equation}
        |\poploss_{\xydist_d}(f) - \emloss{S}{f}| \le O\left(M^3\sqrt{\frac{ \log{1/\delta}}{|S|}}\right)
    \end{equation}
    with probability at least $1-\delta$ over samples $S$ drawn i.i.d. from $\xydist_d$.
\end{lemma}
\begin{proof}
    We apply the properties of Rademacher complexity (see for example Theorem 12 in \cite{Bartlett_Mendelson_2001} and Theorem 4.10 in \cite{wainwright2019high}) to give an estimation error bound over $\setofnns{3}^M$ as follows. Define the loss class $\gL_{\setofnns{3}^M,f_{d}}:= \{(h - f_d)^2 : h \in \setofnns{3}^M\}$. With probability at least $1-\delta$, 
    \begin{align}
        \sup&_{h\in\setofnns{3}^M} |\poploss_{\xydist_d}(h) - \emloss{S}{h}| \\
        &\le O\left(\Rad_{\gX_d^m}(\gL_{\setofnns{3}^M,f_{d}}) 
        + \sqrt{\frac{\log(1/\delta)}{m}} 
        \sup_{h \in \setofnns{3}^M} \|f_d - h\|_{L^\infty}^2  \right)\\
        &\le O\left(\sup_{h\in\setofnns{3}^M}(\|f_d - h\|_{L^\infty})\left( \Rad_{\gX_d^m}(\setofnns{3}^M) + 1/\sqrt{m}\right) 
        + \sqrt{\frac{\log(1/\delta)}{m}} 
        \sup_{h \in \setofnns{3}^M} \|f_d - h\|_{L^\infty}^2\right).
    \end{align}
    Plugging in the bounds from \Cref{lem:rademacher complexity bound,lem:uniform loss bound}, this becomes
    \begin{equation}
        \sup_{h\in\setofnns{3}^M} |\poploss_{\xydist_d}(h) - \emloss{S}{h}| 
        = O\left(M^3\sqrt{\frac{\log{1/\delta}}{m}}\right).
    \end{equation}
\end{proof}

\subsection{Full Proof of No Reverse Depth Separation}
\label{sec:full proof of no rev dep sep}

\subsubsection{Proof of \Cref{lem:learning with two layers means approximate with small cost}}
\label{sec:proof learning with two layers means approximate with small cost}
\begin{proof}
    Fix $\varepsilon> 0$.
    Let $S$ be a sample from $\xydist_d$ of size $m_2\left(\frac{\varepsilon}{2}\right)$. 
    As in the proof of \Cref{thm:depth separation} Part 1, we rely on the existence of an interpolant. 
    By \Cref{lem:tent interpolator is cheap},
    with probability at least $0.6$
    there is an interpolant $\hat f_S \in \setofnns{2}$ of the samples $S$ with 
    $
         \Repregbias{2}\left(\hat f_S\right) 
         \le 100 \sqrt{2}m_2\left(\frac{\varepsilon}{2}\right)^{\frac{d+3}{d-1}}.
    $
    Because $\idealrule{2}(S) \in \pareto{2}$ is Pareto optimal,
    it follows that
    $
        \Repregbias{2}(\idealrule{2}(S)) \le  \Repregbias{2}(\hat f_S).
    $
    We conclude that
    \[
        \P\left(
            \Repregbias{2}(\idealrule{2}(S)) > 
            100 \sqrt{2}m_2\left(\frac{\varepsilon}{2}\right)^{\frac{d+3}{d-1}}
        \right) 
        \le 0.4.
    \]
    
    On the other hand, since $\E_{S}[\poploss_{\xydist_d}(\idealrule{2}(S))] \le \frac{\varepsilon}{2}$
    whenever $|S| \ge m_2\left(\frac{\varepsilon}{2}\right)$, it follows from Markov's inequality that
    \begin{equation}
        \P\left(\poploss_{\xydist_d}(\idealrule{2}(S)) > \varepsilon\right)  \le 0.5.
    \end{equation}
    Therefore,
    \begin{equation}
        \P\left(
            \poploss_{\xydist_d}(\idealrule{2}(S)) > \varepsilon 
            \text{~or~}
            \Repregbias{2}(\idealrule{2}(S)) > 
            100 \sqrt{2}m_2\left(\frac{\varepsilon}{2}\right)^{\frac{d+3}{d-1}}
        \right)
        \le 0.9 < 1.
    \end{equation}
    We conclude that there is some sample $S_{\varepsilon}$ from $\xydist_d$ of size $m_2\left(\frac{\varepsilon}{2}\right)$ such that 
    \begin{equation}
        \poploss_{\xydist_d}(\idealrule{2}(S_{\varepsilon})) \le \varepsilon
        \text{~and~} 
        \Repregbias{2}(\idealrule{2}(S_{\varepsilon})) \le 
        100 \sqrt{2}m_2\left(\frac{\varepsilon}{2}\right)^{\frac{d+3}{d-1}}.
    \end{equation}
    We choose $f_{\varepsilon} = \idealrule{2}(S_{\varepsilon})$.
\end{proof}

\subsubsection{Proof of \Cref{thm:no reverse depth separation} (No Reverse Depth Separation)}
\label{sec:proof details no rev dep sep}
\begin{proof}
    Fix $\varepsilon,\delta>0$ and $\alpha \ge 1$. Let $\theta = \frac{\varepsilon}{2\alpha }$. 
    Under the assumptions of the theorem,
    \Cref{lem:learning with two layers means approximate with small cost} tells us there is a function $f_\theta \in \setofnns{2}$
    such that $\poploss_{\xydist_d}(f_\theta) \le \theta/8$ and
    $
        \Repregbias{2}(f_\theta) \le O\left(
            m_2\left(\frac{\varepsilon}{32\alpha }\right)^{\frac{d+3}{d-1}}
        \right).
    $
    Let 
    \begin{equation}
        \twolayerwidth 
        = \frac{24R_2(f_\theta)^2}{\theta}
        = O\left(\frac{m_2\left(\frac{\varepsilon}{32\alpha }\right)^{\frac{2(d+3)}{d-1}}\alpha }{\varepsilon}\right).
    \end{equation}
    \Cref{lem:can approximate with narrow network with same R2 cost} allows us to approximate $f_\theta$ --- and thus $\xydist_d$ --- with width $\twolayerwidth$; there is some $\tilde f_\theta \in \setofnns{2,\twolayerwidth}$ such that 
    $
        \Repregbias{2}(\tilde f_\theta;\twolayerwidth) \le \Repregbias{2}(f_\theta)
    $
    and
    $\|f_\theta - \tilde f_\theta\|_{L^2} < \sqrt{\theta/8}$.
    Thus,
    \begin{equation}
        \poploss_{\xydist_d}(\tilde f_\theta)
        \le 2 \left( \poploss_{\xydist_d}(f_\theta) + \|f_\theta - \tilde f_\theta\|_{L^2}^2\right)
        \le \theta/2.
    \end{equation}
    If $\width \ge \max(\twolayerwidth, 4d)$, then \Cref{lem:ub on R3b by R2b with identity layer} tells us that $\tilde f_\theta \in \setofnns{3,\width}$ and
    \begin{align}
        \Repregbias{3}(\tilde f_\theta;\width) 
        &\le \frac{4d}{3} + \frac{4}{3} \Repregbias{2}(\tilde f_\theta;\twolayerwidth) \\
        &\le \frac{4d}{3} + \Repregbias{2}(f_\theta) \\
        &= O\left(
            d + m_2\left(\frac{\varepsilon}{32\alpha }\right)^{\frac{d+3}{d-1}}
        \right). \label{eq:ftheta cost bound}
    \end{align}

    By the estimation error bound in \Cref{lem:estimation error bound} and the union bound, with probability at least $1-\delta$ we have that
    \begin{equation}
    \label{eq:estimation error bound for real rule}
        \left|\emloss{S}{\alpharealrule{3,\width}(S)} - \poploss_{\xydist_d}(\alpharealrule{3,\width}(S))\right|
        = O\left(
            \sqrt{
                \frac{
                    \Repregbias{3}( \alpharealrule{3,\width}(S) ;\width)^6 \log(1/\delta)
                }
                {|S|}
            }
        \right)
    \end{equation}
    and
    \begin{equation}
    \label{eq:estimation error bound for ftheta}
        \left|\emloss{S}{\tilde f_\theta} - \poploss_{\xydist_d}(\tilde f_\theta)\right|
        = O\left(
            \sqrt{
                \frac{
                    \Repregbias{3}( \tilde f_\theta ;\width)^6 \log(1/\delta)
                }
                {|S|}
            }
        \right).
    \end{equation}
    If $|S| \ge m_3(\varepsilon,\delta,\alpha)$, where 
    \begin{equation}
        m_3(\varepsilon,\delta,\alpha) = 
        O\left(\frac{\alpha^6 \left(d + m_2\left(\frac{\varepsilon}{64}\right)^{\frac{d+3}{d-1}} \right)^6 \log{1/\delta}}{\varepsilon^2}\right),
    \end{equation}
    then
    \Cref{eq:ftheta cost bound,eq:estimation error bound for ftheta} imply that $\left|\emloss{S}{\tilde f_\theta} - \poploss_{\xydist_d}(\tilde f_\theta)\right| \le \theta/2$, and so $\emloss{S}{\tilde f_\theta} \le \theta$. Hence 
    \begin{align}
        \Repregbias{3}(\alpharealrule{3,\width}(S);\width) 
        &\le \alpha  \inf_{\substack{f \in \setofnns{3,\width} \\ \emloss{S}{f} \le \theta}} \Repregbias{3}(f;\width) \\
        &\le \alpha  \Repregbias{3}( \tilde f_\theta ;\width ) \\
        &= O\left(
            \alpha  
            \left(
                d + m_2\left(\frac{\varepsilon}{32\alpha }\right)^{\frac{d+3}{d-1}}
            \right)
        \right). \label{eq:realrule cost bound}
    \end{align}
    By \Cref{eq:estimation error bound for real rule,eq:realrule cost bound,}, if 
    $
        |S| \ge 
        m_3(\varepsilon,\delta,\alpha)
    $ 
    then 
    $
        \left|\emloss{S}{\alpharealrule{3,\width}(S)} - \poploss_{\xydist_d}(\alpharealrule{3,\width}(S))\right|
        \le \frac{\varepsilon}{2}.
    $
    Therefore
    $
        \poploss_{\xydist_d}(\alpharealrule{3,\width}(S))
        \le \alpha  \theta + \frac{\varepsilon}{2} = \varepsilon.
    $
\end{proof}

\end{document}